\documentclass[conference,compsoc]{IEEEtran}
\ifCLASSOPTIONcompsoc
\IEEEoverridecommandlockouts
\usepackage[nocompress]{cite}
\else
\usepackage{cite}
\fi
\usepackage{textcomp}
\usepackage{stfloats}
\usepackage{url}
\usepackage{verbatim}
\usepackage{graphicx}
\usepackage{amsthm}
\usepackage{optidef}
\usepackage{subcaption}
\usepackage{xcolor}
\usepackage{times}
\usepackage{amsmath}
\usepackage{amssymb}
\usepackage{amsfonts}
\usepackage{algorithmic}
\usepackage{multirow}
\usepackage[linesnumbered,ruled]{algorithm2e}
\usepackage{hyperref}
\usepackage{enumitem}

\allowdisplaybreaks

\DeclareMathOperator{\Var}{Var}

\DeclareMathOperator{\Clip}{Clip}

\DeclareMathOperator{\Unif}{Unif}
\DeclareMathOperator{\range}{Range}

\newtheorem{thm}{Theorem}

\newtheorem{defi}{Definition}
\newtheorem{lem}{Lemma}

\newcommand{\norm}[1]{\left\lVert#1\right\rVert}

\hypersetup{
	colorlinks,
	linkcolor={blue!70!green},
	citecolor={green!70!blue},
	urlcolor={orange!70!red}
}

\newcommand{\mypara}[1]{\noindent\textbf{#1.}\xspace}

\ifCLASSINFOpdf
\else
\fi
\begin{document}

\date{}

\title{\Large \bf Consistent Estimation of Numerical Distributions under Local Differential Privacy by Wavelet Expansion}

\author{\IEEEauthorblockN{Puning Zhao$^{1,4}$,
		Zhikun Zhang$^{2\#}$,
		Bo Sun$^2$, 
		Li Shen$^1$,
		Liang Zhang$^1$
		Shaowei Wang$^3$,
		Zhe Liu$^2$}
	\IEEEauthorblockA{$^1$ Shenzhen Campus of Sun Yat-sen University\ $^2$ Zhejiang University\ $^3$ Guangzhou University\\ $^4$ Guangdong Key Laboratory of Information Security Technology}
	\thanks{\# Corresponding authors}
}

\maketitle

\begin{abstract}
Distribution estimation under local differential privacy (LDP) is a fundamental and challenging task. Significant progresses have been made on categorical data. However, due to different evaluation metrics, these methods do not work well when transferred to numerical data. In particular, we need to prevent the probability mass from being misplaced far away. In this paper, we propose a new approach that express the sample distribution using wavelet expansions. The coefficients of wavelet series are estimated under LDP. Our method prioritizes the estimation of low-order coefficients, in order to ensure accurate estimation at macroscopic level. Therefore, the probability mass is prevented from being misplaced too far away from its ground truth. We establish theoretical guarantees for our methods. Experiments show that our wavelet expansion method significantly outperforms existing solutions under Wasserstein and KS distances.
\end{abstract}

\section{Introduction}

Local differential privacy (LDP) \cite{dwork2006calibrating} is the de facto standard for data privacy, which has been employed by a number of high-tech companies including Apple \cite{apple}, Google \cite{erlingsson2014rappor} and Microsoft \cite{ding2017collecting}. In these applications, a fundamental task is distribution estimation, which is not only directly applicable in its own right \cite{yang2024local}, but also serves as a building block for other tasks, such as range query \cite{gu2019supporting}. Consequently, in recent years, the design of efficient distribution estimators under LDP has received widespread attention \cite{erlingsson2014rappor,bassily2015local,kairouz2016discrete,wang2017locally,acharya2019hadamard,ye2018optimal,wang2020locally,fang2023locally}.

Most existing works for LDP frequency estimation focus on categorical data. However, in many practical applications, attributes are actually numerical, such as income, age, height, location, etc\cite{christin2016privacy}. The metrics for evaluating the estimation of numerical data are crucially different from those for categorical data. An example is shown in \autoref{fig:example}, in which there are two estimated probability density functions (pdf) $\hat{f}_1$ and $\hat{f}_2$ that both approximate a uniform distribution. Traditional metrics for categorical attributes such as $\ell_1$ distance and Kullback-Leibler divergence assign equal scores to $\hat{f}_1$ and $\hat{f}_2$. However, in practical applications, such as range query, $\hat{f}_2$ is considered as a significantly worse estimate than $\hat{f}_1$. As shown in \autoref{fig:example}, for $\hat{f}_2$, the probability mass is misplaced to a longer distance, thus the error of cumulative distribution function (cdf) estimate $\hat{F}_2$ is much larger than that of $\hat{F}_1$. As discussed in some recent works \cite{li2020estimating,li2021trade,du2024numerical,feldman2024instance}, Wasserstein distance \cite{rabin2012wasserstein} and Kolmogorov-Smirnov (KS) distance \cite{massey1951kolmogorov} are more suitable. The change of metrics poses challenge to the algorithm design. To minimize Wasserstein or KS distances, we need to avoid the misplacement of probability mass for long distances.

\begin{figure}[h!]
	\includegraphics[width=\linewidth]{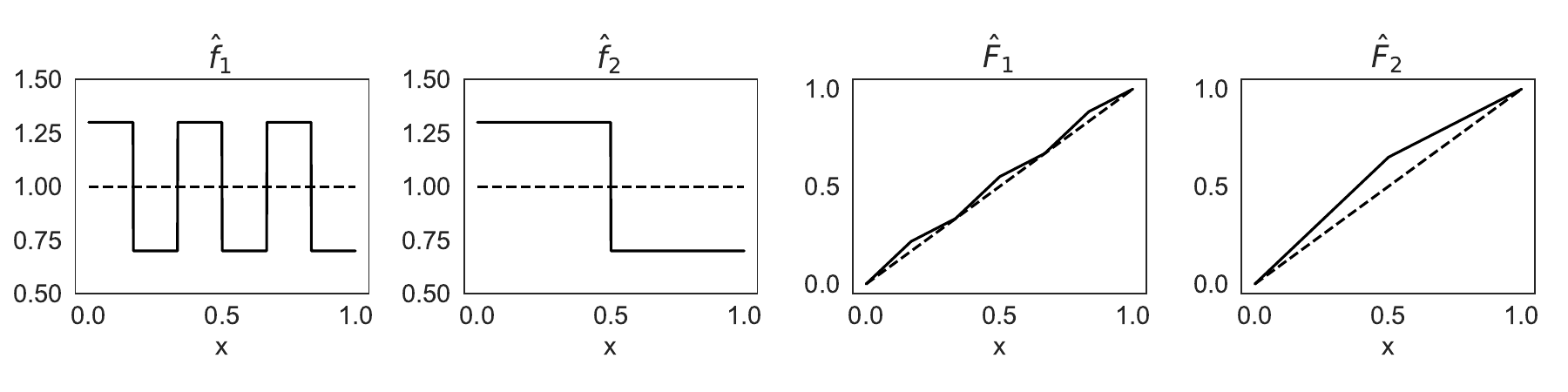}
	\caption{An example of two estimated distributions.}\label{fig:example}	
\end{figure}

\mypara{Existing solutions}
A solution that comes to mind immediately is to convert numerical variables to categorical variables by binning. Various standard frequency oracles can be used to estimate the probabilities of these bins, such as $k$-Randomized Response (kRR) \cite{warner1965randomized,kairouz2016discrete}, Rappor \cite{erlingsson2014rappor}, and Optimized Unary Encoding (OUE) \cite{wang2017locally}. However, the binning approach does not yield desirable performance. Intuitively, the error of cdf estimation increases with the accumulation of estimation errors for each bin. As a result, the final estimated distribution may misplace some probability mass far away from its ground truth, thereby resulting in large estimation error under Wasserstein or KS distances. Another drawback is that the optimal number of bins is unknown before observing the data, thus the parameter tuning introduces additional privacy cost.

Motivated by the drawbacks of categorical frequency oracles, in recent years, several new algorithms are proposed to handle numerical attributes in a better way. The general idea is to add noise to each sample to satisfy the LDP requirement, and then calculate the empirical distribution of these perturbed samples. After that, these algorithms attempt to recover the original distribution using some numerical tools. \cite{li2020estimating} proposes the square wave (SW) mechanism as the local randomizer, and then use Expectation-Maximization (EM) to recover the true distribution. \cite{fang2023locally} recovers the underlying distribution using some deconvolution algorithms based on Wiener filter. Although the sample distribution is blurred by the additive noise, the probability mass is misplaced to neighboring places. Therefore, compared with categorical frequency oracles, these methods yield more desirable results. However, for non-smooth distributions, such blurring inevitably results in the loss of some detailed information, which negatively impacts the overall accuracy. Moreover, the high computational overhead of recovery algorithms such as EM and deconvolution hinders their practical application.


\mypara{Our proposal} In this paper, we propose a wavelet expansion approach to frequency estimation under LDP. The wavelet basis consists of a series of wavelet functions. Wavelet system is a complete orthonormal basis, thus any distribution can be approximated arbitrarily close by wavelet expansion. The challenge is how to estimate the coefficients under LDP requirements. Since the privacy budget and the sample size are both limited, we can not get an accurate estimation of all coefficients. Therefore, we prioritize the estimation of low-order coefficients, as they have more impact on the final Wasserstein and KS distances. To be more precise, a wavelet expansion up to $J$-th order requires the coefficient estimation of $a_{jk}$ for $j\in \{0,1,\ldots, J\}$. Our algorithm ensures that the variance of $a_{jk}$ to be small with small $j$.

Intuitively, our method performs well because it aligns well with the goal of distribution estimation for numerical data. The wavelet expansion at low orders describes the macroscopic information of the distribution, while high-order coefficients describes the details. Since the estimation of $a_{jk}$ is sufficiently accurate for small $j$, only a small amount of probability mass is misplaced far away. If $j$ is large, then the estimation of $a_{jk}$ may be inaccurate, but coefficients with large $j$ are relatively less important, since the probability mass misplaced to the vicinity does not contribute significantly to the Wasserstein or KS distance. Moreover, compared with SW or filter based approaches, our method does not blur the distribution, thus for non-smooth distributions, the wavelet approach preserves the details in a better way, which further improves the accuracy.

We evaluate our method both theoretically and empirically. Given the sample size $n$, the privacy budget $\epsilon$ and the level $J$, our theoretical analysis gives an upper bound of the Wasserstein distance between estimated distribution and the ground truth. The optimal value of $J$ can then be obtained by minimizing the upper bound. Therefore, the theoretical analysis not only provides the convergence guarantee, but also provides a guideline for parameter tuning. The selection of $J$ does not introduce additional privacy cost, since it depends entirely on $n$ and $\epsilon$, which is known before observing the data. In general, $J$ is larger with large $n$ or $\epsilon$. We also run experiments on both synthesized and real data. The results show that our method consistently outperforms existing approaches, especially for non-smooth datasets. 

The contributions of this paper are summarized as follows:

\begin{itemize}
	\item We propose a wavelet expansion approach to estimating numerical distributions under LDP.
	
	\item We provide theoretical analysis, which shows that our method achieves consistent distribution estimation of numerical data under Wasserstein and KS distances. 
	
	\item We conduct extensive experiments on both synthesized and real-world datasets. The results validate the effectiveness of our proposed approach.
\end{itemize}

\mypara{Roadmap} In \autoref{sec:prelim}, we provide some background information. The exact problem definition and existing solutions are shown in \autoref{sec:formulation}. In \autoref{sec:method}, we show the proposed method. The theoretical analysis is then shown in \autoref{sec:theory}. After that, \autoref{sec:numerical} presents the results of numerical experiments. Finally, we discuss related work in \autoref{sec:related} and conclude the work in \autoref{sec:conc}.

\section{Preliminaries}\label{sec:prelim}

\subsection{Local Differential Privacy}
To begin with, we revisit the concept of LDP. Assume that there are $n$ users and one aggregator. Each user has a value $x\in \mathcal{D}$. To protect the privacy of users, $x$ is randomized before sending to the server. Denote $Q$ as the function for such randomization. LDP is defined as follows.

\begin{defi}\label{def:ldp}
	(LDP) A local randomizer $Q$ is $\epsilon$-LDP if for every $x$, $x'$ and $S\subseteq \range(Q)$,
	\begin{eqnarray}
		\text{P}(Q(x)\in S)\leq e^\epsilon \text{P}(Q(x')\in S),
	\end{eqnarray}
	in which $\range(Q)$ denotes the set of all possible outputs of $Q$. The randomness comes from the local randomizer $Q$.
\end{defi} 
The user does not report the local sample $x$ to the server. Instead, it only sends $Q(x)$, thus the privacy of users is protected even if the server is malicious.

\subsection{Wavelet}

Wavelets play an important role in audio and video signal processing \cite{zhang2019wavelet}. By applying wavelet transform, long signals can be compressed to much smaller ones, while still retaining enough information. Therefore, the restored signal is very close to the original one. 

A wavelet system consists of a \emph{scaling function} $\phi:[0,1]\rightarrow \mathbb{R}$ and a series of \emph{wavelet functions} $\{\psi_{jk}| j\in \{0,\ldots, J\},k\in \{0,\ldots, 2^J - 1 \}\}$. For all $j$ and $k$, $\psi_{jk}$ satisfies 
\begin{eqnarray}
	\psi_{jk}(x)=2^\frac{j}{2}\psi(2^j x-k),
	\label{eq:psijk}
\end{eqnarray}
in which $\psi:[0,1]\rightarrow \mathbb{R}$ is called \emph{mother wavelet function}.

A wavelet system in $L^2([0,1])$ is complete. Each function $f\in L^2([0,1])$ can be arbitrarily closely approximated by wavelet expansions. Moreover, a wavelet system is also orthonormal, i.e. 
\begin{eqnarray}
	\int_0^1 \psi_{jk}(x)\psi_{lm}(x)dx = \delta_{jl}\delta_{km},
	\label{eq:orthonormal}
\end{eqnarray}
in which $\delta$ is the Kronecker delta, such that $\delta_{jl}=1$ if $j=l$, otherwise $\delta_{jl} = 0$.

The \emph{wavelet expansion} of a function is
\begin{eqnarray}
	f_J(x) = \phi(x)+\sum_{j=0}^J \sum_{k=0}^{2^j - 1}a_{jk} \psi_{jk}(x).
	\label{eq:exp}
\end{eqnarray}

In this paper, for simplicity, we use the Haar wavelet. The scaling function and the mother wavelet function are defined as follows:
\begin{eqnarray}
	\phi(t) &=& \left\{
	\begin{array}{cc}
		1 &\text{if } 0\leq t\leq 1\\
		0 &\text{otherwise.}
	\end{array}
	\right.\\
	\psi(t) &=&\left\{
	\begin{array}{cc}
		1 &\text{if }  0\leq t <\frac{1}{2}\\
		-1 &\text{if }  \frac{1}{2}\leq t \leq 1\\
		0 &\text{otherwise.}
	\end{array}
	\right.
\end{eqnarray}
Correspondingly, $\psi_{jk}$ is defined according to \eqref{eq:psijk}, thus
\begin{eqnarray}
	\psi_{jk}(t)=\left\{
	\begin{array}{ccc}
		2^\frac{j}{2} &\text{if} & k2^{-j}\leq t<(k+\frac{1}{2})2^{-j}\\
		-2^\frac{j}{2} &\text{if} & (k+\frac{1}{2})2^{-j}\leq t\leq (k+1)2^{-j}\\
		0&\text{if} & t\notin [k2^{-j}, (k+1)2^{-j}].
	\end{array}
	\right.
	\label{eq:haarjk}
\end{eqnarray}
It is straightforward to verify that the Haar wavelet functions $\{\phi_{jk} |j\in \{1,\ldots, J\}, k\in \{0,\ldots, 2^j -1 \} \}$ satisfy the orthonormal condition \eqref{eq:orthonormal}.
\section{Problem Formulation and Existing Solutions}\label{sec:formulation}
\subsection{Problem Formulation}
This paper focus on distribution estimation problems for numerical data. Given a dataset which has $n$ samples $X_1,\ldots, X_n$, define the empirical cumulative distribution function (cdf) is defined as
\begin{eqnarray}
	F(x) = \frac{1}{n}\sum_{i=1}^n \mathbf{1}(X_i\leq x).
	\label{eq:cdfreal}
\end{eqnarray}
Our goal is to obtain an estimate $\hat{F}$ satisfying $\epsilon$-LDP.

For simplicity, in this paper, we assume that samples are all supported in a fixed domain $\mathcal{D}=[0,1]$. The analysis in this paper can be easily generalized to arbitrary distributions with bounded support.

Following \cite{li2020estimating}, we use the following two metrics to evaluate the quality of estimation.

\mypara{Wasserstein distance (aka. Earth mover distance)} Wasserstein distance arises from the idea of the optimal transport problem, which is one of foundational problems of optimization \cite{villani2009optimal}. Intuitively, Wasserstein distance is the minimum cost of moving the probability mass from one distribution to another. This paper uses $W_1$ distance in one dimension. Formally, for two distributions $P$ and $Q$, the Wasserstein distance is defined as
\begin{eqnarray}
	W(P,Q) = \inf_{\pi \in \mathcal{M}(P, Q)} \mathbb{E}_{(U, V)\sim \pi}[\norm{U-V}_1],
\end{eqnarray} 
in which $\mathcal{M}(P, Q)$ denotes the set of all joint distributions with marginals $P$ and $Q$, respectively.

For two one-dimensional distributions with cumulative distribution function (cdf) $F_1$ and $F_2$, it can be easily shown that the Wasserstein distance between them is
\begin{eqnarray}
	W(F_1, F_2) = \int |F_1(x)-F_2(x)| dx.
    \label{eq:Wdf}
\end{eqnarray}

\mypara{Kolmogorov-Smirnov (KS) distance} KS distance is the maximum difference of cdf between two distributions:
\begin{eqnarray}
	KS(F_1,F_2) = \sup_{x\in \mathcal{D}} |F_1(x)-F_2(x)|.
    \label{eq:KSdf}
\end{eqnarray}

Unlike traditional measures such as $\ell_1$ distance and KL divergence, Wasserstein and KS distances are determined not only by the amount of probability mass being misplaced, but also the distance of such misplacement. A large distance is heavily penalized. Therefore, Wasserstein and KS distances are more suitable for numerical data, as these two metrics reflect the ordinal nature of numerical data in a better way.  

\mypara{Relationship to range query problems} In range query problems, the goal is to estimate the fraction of samples falling in $[a,b]$. Wasserstein distance gives an upper bound of the average error of range query. It can be shown that if $a$ and $b$ are randomly generated from $[0,1]$ with a uniform distribution ($a$ and $b$ are swapped if $a>b$), given an estimated cdf $\hat{F}$, then the expectation of the absolute error of range query is no more than $2W(\hat{F}, F)$. KS distance reflects the maximum error of range query. Given an estimated cdf $\hat{F}$, it can be easily shown that the maximum absolute error of range query over all $a,b\in [0,1]$ does not exceed $2KS(\hat{F}, F)$.

\subsection{Categorical Frequency Oracles}\label{sec:cfo}
Categorical frequency oracles are algorithms for frequency estimation problem with the support $\mathcal{D}=[d]$ being a categorical domain, with $d$ being the alphabet size. Here we introduce two commonly used categorical frequency oracles.

\mypara{$k$-Randomized Response (kRR)} Given the input $x\in \mathcal{D}$, in which $\mathcal{D}=[d]$, the local randomizer generates output $Y$ according to the following distribution:
\begin{eqnarray}
	\text{P}(Y=j|x) = \left\{
	\begin{array}{cc}
		\frac{e^\epsilon}{e^\epsilon+d-1} &\text{if }  x = j\\
		\frac{1}{e^\epsilon+d-1} &\text{otherwise.}
	\end{array}
	\right.
\end{eqnarray}
Given $n$ input samples $x_1,\ldots, x_n$, the server receives the reported values $Y_1,\ldots, Y_n$. Then the frequency estimate is
\begin{eqnarray}
	\hat{\mu}_j = \frac{\frac{1}{n}\sum_{i=1}^n \mathbf{1}(Y_i=j) - q}{p-q},
\end{eqnarray}
for $j=1,\ldots, d$, in which $p=e^\epsilon/(e^\epsilon+d-1)$, and $q=1/(e^\epsilon+d-1)$. Theoretical analysis has shown that kRR is suitable for small $d$ or large $\epsilon$ \cite{wang2017locally,wang2018locally}. With large $d$ or small $\epsilon$, kRR is relatively inefficient.

\mypara{Optimized Unary Encoding (OUE)} Given $x\in [d]$, the local randomizer generates a vector $\mathbf{Y}\in \{0,1\}^d$ according to the following rule:
\begin{eqnarray}
	\text{P}(Y(j) = 1|x) = \left\{
	\begin{array}{cc}
		\frac{1}{2} & \text{if } x = j\\
		\frac{1}{e^\epsilon+1} &\text{otherwise,}
	\end{array}
	\right.
\end{eqnarray}
in which $Y(j)$ is the $j$-th element of $\mathbf{Y}$. Then the frequency estimate is
\begin{eqnarray}
	\hat{\mu}_j = \frac{\frac{1}{n}\sum_{i=1}^n Y(j) - q}{p-q},
\end{eqnarray}
in which $p=1/2$, $q=1/(e^\epsilon+1)$. The value of $p$ and $q$ are selected to minimize the variance, while ensuring that $\hat{\mu}_j$ is unbiased and satisfy $\epsilon$-LDP \cite{wang2017locally}. OUE is relatively more suitable for small $\epsilon$ or large $d$.

\mypara{Optimized Local Hashing (OLH) \cite{wang2017locally}} OLH randomizes samples with Hash functions. Compared with OUE, OLH has the same variance, lower communication cost but higher computation cost. We omit the details here.

\mypara{Subset selection \cite{ye2018optimal,wang2019local}} Given the input $x\in [d]$, the subset selection randomizer generates $\mathbf{Y}\in \{0,1\}^d$. It has an adjustable parameter $m$. Among these $d$ elements, $m$ of them have value $1$. The distribution is specified as follows:
\begin{eqnarray}
	\text{P}(\mathbf{Y}=\mathbf{y}|x=j)=\left\{
	\begin{array}{ccc}
		\frac{e^\epsilon}{\Omega} &\text{if} & y(j)=1\\
		\frac{1}{\Omega} &\text{if} & y(j)=0.
	\end{array}
	\right.
\end{eqnarray}
According to the theoretical analysis in \cite{ye2018optimal}, with optimally tuned $m$, the subset selection achieves optimal estimation error under all privacy regimes. The optimal value of $m$ is large under small $\epsilon$. With $\epsilon>\ln d$, the optimal $m$ becomes $1$, and the subset selection method reduces to kRR.

\mypara{Application of categorical frequency oracles in numerical domains} For cases with samples following a continuous distribution, a natural approach is to divide the domain into bins, and then estimate the frequency of each bin. The estimated cdf of the distribution can then be constructed by the estimated frequencies. To be more precise, define a function $c:[0,1]\rightarrow [d]$ that maps the continuous value into the index of the assigned bin. Denote $h$ as the bin size, then
	$c(x) = \left\lfloor x/h\right\rfloor$.
In other words, all samples within $[kh, (k+1)h)$ are assigned into the $k$-th bin, $k=1,\ldots, d$. After passing categorical frequency oracles like kRR and OUE, with some normalization step such as \cite{wang2020locally}, we get an estimated frequency $(\hat{\mu}_1,\ldots, \hat{\mu}_d)$. Then we can obtain the estimated cdf at the connecting points:
\begin{eqnarray}
	F(kh) = \sum_{j=1}^k \hat{\mu}_j.
	\label{eq:cdfbin}
\end{eqnarray} 
For $x\notin \{kh|k=1,\ldots, d \}$, $F(x)$ can be obtained by interpolation.

\mypara{Drawbacks of categorical frequency oracles with binning} For categorical data, kRR and OUE are well performed distribution estimators. However, due to the ordinal nature of numerical data, the accuracy of categorical frequency oracles with binning becomes undesirable. Recall \eqref{eq:cdfbin}, with large $k$, the estimation error of $F(kh)$ can be serious since the error $\hat{\mu}_j$ accumulates from $j=1$ to $k$. In particular, the estimation error of $F(1/2)$ can be quite large, which indicates that the probability mass is misplaced from the left half of the domain to the right, or vice versa. Such a misplacement of probability mass for a long distance can be serious for numerical data. As a result, the performance is not satisfactory under Wasserstein and KS distances. 

Moreover, the selection of bin sizes is hard and may introduce additional privacy cost. The estimation error comes from the bias due to grouping all samples within a bin together, as well as the variance from local randomization. A small bin size leads to small bias but larger variance, and vice versa. Therefore, the bin size needs to be selected to achieve a good tradeoff, and the optimal size depends on the smoothness of distribution, which is unknown in advance. If the bin size is tuned experimentally, then some additional privacy cost is inevitable.

\subsection{Randomization in Numerical Domain}
Motivated by the drawbacks of categorical frequency oracles, some recent works propose methods to exploit the ordinal nature of numerical data in a better way. An intuitive idea is to let each user add some noise to the value of local sample $x$ to satisfy the LDP requirement, and then calculate the empirical distribution of these noisy samples. The real distribution can then be inferred from the distribution of noisy samples. 

\mypara{Square wave (SW) \cite{li2020estimating}} The general idea is that for different $x$ and $x'$, \cite{li2020estimating} designs a local randomizer such that the Wasserstein distance between the output distributions given $x$ and $x'$ as input values are maximized, in order to make it easier to distinguish them. The conditional probability of output $y$ given input $x$ is constructed as follows.
\begin{eqnarray}
	f_{Y|X}(y|x) = \left\{
	\begin{array}{cc}
		p &\text{if } |y-x|\leq b\\
		q &\text{otherwise,}
	\end{array}
	\right.
\end{eqnarray}
in which $p=e^\epsilon /(2be^\epsilon+1)$, $q=1/(2be^\epsilon+1)$. In \cite{li2020estimating}, $b$ is selected to maximize the mutual information $I(Y;X)$ between the input and the output, in order to achieve minimal loss of information \cite{cover1999elements}. As a result, the optimal $b$ is
\begin{eqnarray}
	b = \frac{\epsilon e^\epsilon - e^\epsilon+1}{2e^\epsilon (e^\epsilon - 1 - \epsilon)}.
\end{eqnarray}

The randomization with SW mechanism makes the distribution more blurred. To recover the ground truth, \cite{li2020estimating} conducts a post-processing with EM algorithm. 

\mypara{Wiener filter \cite{fang2023locally}} This approach adds a fixed noise $\mathbf{n}$ to each sample to satisfy the $\epsilon$-LDP requirements. After adding noise, the sample distribution becomes the convolution of the original distribution with a fixed kernel vector $h$, which is determined by the noise. The final step is to recover the original distribution using a deconvolution algorithm. In \cite{fang2023locally}, two algorithms are designed, including the Direct Wiener (DW) filter and Improved Iterative Wiener (IIW) filter.

\mypara{Drawbacks of numerical randomization approaches} Firstly, the recovery of original distribution from the distribution of randomized samples is computationaly expensive. With samples divided into $d$ bins, each iteration of EM algorithm requires $O(d^3)$ time. Secondly, these approaches exhibits worse performance for non-smooth distributions. Intuitively, the distribution of samples is blurred after adding noise. Therefore, for non-smooth distributions, a lot of detailed information will be lost due to the LDP mechanism. 

\section{Our Proposal}\label{sec:method}

\subsection{Intuition}

We propose a wavelet expansion approach. As discussed earlier, a wavelet system $\{\psi_{jk}\}_{j,k\in \mathbb{Z}}
$ forms a complete and orthonormal basis for $L^2([0,1])$. Therefore, the real distribution can be arbitrarily closely approximated by wavelet series. It remains to estimate the wavelet coefficients.

The sample size and the privacy budget are both limited, thus we can not expect to achieve an accurate estimation of all coefficients. Our method prioritizes low-order coefficients estimates, while the accuracy of high-order coefficients can be appropriately sacrificed. Low-order coefficients are more important, because an inaccurate estimation of low-order coefficients results in the misplacement of probability mass with larger distance. 

The whole procedures of the proposed method are summarized in Algorithm \ref{alg:all}.
\begin{algorithm}[t]
	\caption{The Wavelet Expansion Algorithm for Frequency Estimation}\label{alg:all}
	\textbf{Input:} Dataset $x_1, \ldots, x_n$\\
	\textbf{Output:} Estimated pdf $\hat{f}$\\
	\textbf{Parameter:} $J$
	\begin{algorithmic}[1]
		\STATE For $j=0,\ldots, T$, construct $S_j=\emptyset$
		\STATE Allocate samples into $S_0,\ldots, S_j$ following rules described in \autoref{sec:allocation}
		\FOR{$j=0,\ldots, j$}
		\FOR{$i\in S_j$}
		\STATE \emph{Encoding:} Calculate $\mathbf{v}_i$ with \eqref{eq:encode}
		\STATE Calculate $m$ from \eqref{eq:mopt}, in which $p$, $q$ and $\Omega$ are calculated using Lemma \ref{lem:normalize} and \ref{lem:pq}
		\STATE \emph{Perturbation:} Calculate $\mathbf{Y}_i$ according to Algorithm \ref{alg:perturb}, with input $\mathbf{v}_i$ and parameter $d=2^j$ and $m$
		\ENDFOR
		\STATE \emph{Aggregation:} Calculate $a_{jk}$ using \eqref{eq:ajk}
		\STATE Calculate $f_J$ according to \eqref{eq:fj}
		\ENDFOR
		\RETURN $\hat{f}=f_J$
	\end{algorithmic}
\end{algorithm}
\subsection{Wavelet Expansion}
To begin with, we discuss the wavelet expansion of the cdf of all samples without LDP requirements. After that, we discuss how to approximate such expansion under $\epsilon$-LDP.

Let $\psi_{jk}$ be the Haar wavelet function defined in \eqref{eq:haarjk}. The wavelet expansion of a pdf $f$ using Haar basis up to the $J$-th order can be expressed as follows:
\begin{eqnarray}
	f_J^*(x)=1+\sum_{j=0}^J \sum_{k=0}^{2^j - 1}a_{jk}^*\psi_{jk}(x),
	\label{eq:fjstar}
\end{eqnarray}
in which the coefficient $a_{jk}^*$ is
\begin{eqnarray}
	a_{jk}^* = \frac{1}{n}\sum_{i=1}^n \psi_{jk} (X_i).
	\label{eq:ajkideal}
\end{eqnarray}
Moreover, the corresponding cdf is
\begin{eqnarray}
	F_J^*(x) = \int_0^x f_J^*(t) dt.
	\label{eq:FJstar}
\end{eqnarray}
It can be shown that the distribution constructed above is close to the ground truth. The precise statement is shown in Lemma \autoref{lem:approx}.
\begin{lem}\label{lem:approx}
	For $k=0,1,\ldots, 2^{J+1}$, with $a_{jk}^*$ specified in \eqref{eq:ajkideal}, $F_J^*(k2^{-(J+1)})=F(k2^{-(J+1)})$.
\end{lem}
The proof of Lemma \autoref{lem:approx} is shown in Appendix \ref{sec:approx}. Lemma \autoref{lem:approx} suggests that the cdf reconstructed from the wavelet expansion matches the empirical cdf at $x=k2^{-(J+1)}$ for all $k$. With the increase of $J$, $F_J^*$ matches $F$ at more points. Therefore, if we can calculate $a_{jk}^*$ accurately, then $F$ can be reconstructed at each $k2^{-(J+1)}$, $k=1,\ldots, 2^{J+1}-1$. However, $a_{jk}^*$ is unknown in practice. We need to estimate $a_{jk}^*$ from data under $\epsilon$-LDP.

To begin with, we conduct a user splitting. Samples are randomly divided into $J+1$ groups, named $S_0, \ldots, S_J$. $S_j$ is used for estimating $a_{jk}^*$, $k=0,\ldots, 2^j - 1$. Denote $n_j=|S_j|$ as the number of samples in $S_j$. Then $\sum_{j=0}^J n_j = n$. The precise values of $n_j$, i.e. the allocation of these $n$ samples, will be discussed in \autoref{sec:allocation}.

\subsection{Encoding}\label{sec:encoding}
Now we discuss the encoding step of the $j$-th order. For $i\in S_j$, we encode sample $X_i$ as follows:
\begin{eqnarray}
	\mathbf{v}_i = 2^{-j/2} (\psi_{j0}(X_i), \ldots, \psi_{j, 2^j - 1} (X_i)).
	\label{eq:encode}
\end{eqnarray}

For Haar wavelet, recall \eqref{eq:haarjk}, $\psi_{jk}(X_i)\in \{-2^{j/2}, 0, 2^{j/2}\}$. Therefore, each element of $\mathbf{v}_i$ belongs to $\{-1,0,1\}$. Moreover, \eqref{eq:haarjk} suggests that for any $t\in [0,1]$ $\psi_{jk}(t)\neq 0$ for only one $t$. Therefore, $\mathbf{v}_i$ has only one nonzero element with value being either $1$ or $-1$. Denote $\mathbf{e}_k=(0,\ldots, 1,\ldots, 0)$ as the unit vector whose $k$-th element is $1$, and
\begin{eqnarray}
	\mathcal{V} &=& \{\mathbf{e}_k|k\in\{0,\ldots, 2^j - 1\}\}\nonumber\\
	&&\cup \{-\mathbf{e}_k|k\in\{0,\ldots, 2^j - 1\}\}
\end{eqnarray}
as the input space that contains all encoded vectors, then $\mathbf{v}\in \mathcal{V}$. The set $\mathcal{V}$ contains $2^{j+1}$ elements in total. 

\subsection{Perturbation}
The goal of perturbation step is to generate random signals $\mathbf{Y}_i$ that satisfies the $\epsilon$-LDP requirement. Our encoding step ensures that $\mathbf{v}_i$ has only one nonzero element whose value is either $1$ or $-1$. 

A closely related problem is frequency estimation. As discussed in \autoref{sec:cfo}, there are various methods for frequency estimation. Among these methods, the subset selection method achieves optimal estimation error under all privacy budgets $\epsilon$. Compared with frequency estimation, our task is different since we allow $\mathbf{v}_i$ to have negative elements. Therefore, we make a generalization to the subset selection approach. A straightforward generalization is to convert $\mathcal{V}$ to $\{0,\ldots, 2^{j+1} - 1\}$, and just use the subset selection approach with $d=2^{j+1}$. In this work, we propose a refinement that has better efficiency.

Denote $Q: \mathcal{V} \rightarrow \mathcal{Y}$ as the randomizer, such that the perturbed outputs are $\mathbf{Y}_i=Q(\mathbf{v}_i)$. $Q$ is required to satisfy Definition \ref{def:ldp} under privacy budget $\epsilon$. $\mathcal{Y}$ is the output space of perturbation:
\begin{eqnarray}
	\mathcal{Y} = \left\{\mathbf{y}\in \{-1,0,1\}^d|\norm{\mathbf{y}}_0 = m \right\},
	\label{eq:yset}
\end{eqnarray}
in which $m$ is an adjustable parameter, and $d=2^j$ is the dimensionality of $\mathcal{V}$. Compared with the subset selection method for frequency estimation, the difference is that now $\mathbf{y}$ can take negative values. Now user generates feedback $\mathbf{Y}_i$ that contain $m$ responses given the input encoded vector $\mathbf{v}_i$. From \eqref{eq:yset}, the size of $\mathcal{Y}$ is
\begin{eqnarray}
	|\mathcal{Y}| = \binom{d}{m} 2^m.
\end{eqnarray}

Now it remains to specify the randomizer $Q$. Denote $p_Q$ as the probability mass function (pmf) of randomizer $Q$. $p_Q$ is constructed as follows.
\begin{eqnarray}
	p_Q(\mathbf{y}|\mathbf{v}) = \left\{
	\begin{array}{ccc}
		\frac{e^\epsilon}{\Omega} &\text{if} & \langle \mathbf{y}, \mathbf{v}\rangle = 1\\
		\frac{1}{\Omega} &\text{if} & \langle \mathbf{y}, \mathbf{v}\rangle \neq 1.		
	\end{array}
	\right.
	\label{eq:pq}
\end{eqnarray}
In \eqref{eq:pq}, $\Omega$ is the normalizer that ensures $\sum_{\mathbf{y}\in \mathcal{Y}} p_Q(\mathbf{y}|\mathbf{v}) = 1$. Now we calculate the normalizer $\Omega$. From \eqref{eq:pq}, it can be easily shown that the randomizer $Q$ satisfies $\epsilon$-LDP.

\begin{lem}\label{lem:normalize}
	\eqref{eq:pq} is normalized with
	\begin{eqnarray}
		\Omega = \binom{d-1}{m-1} 2^{m-1}e^\epsilon + \binom{d-1}{m-1}2^{m-1} + \binom{d-1}{m} 2^m.
		\label{eq:norm}
	\end{eqnarray}
\end{lem}
\begin{proof}
	Now we find $\Omega$ such that $\sum_{y\in \mathcal{Y}} p_Q(\mathbf{y}|\mathbf{v}) = 1$. Since $\mathbf{v}$ only has one nonzero element whose value is either $1$ or $-1$, according to \eqref{eq:yset}, $\langle \mathbf{y}, \mathbf{v}\rangle \in \{-1,-,1\}$. To calculate $\Omega$, we need to calculate the number of elements $\mathbf{y}$ in $\mathcal{Y}$ with $\langle \mathbf{y}, \mathbf{v}\rangle = 1$, $\langle \mathbf{y}, \mathbf{v}\rangle = 0$ and $\langle \mathbf{y}, \mathbf{v}\rangle = -1$, respectively.

Without loss of generality, suppose that $\mathbf{v}=\mathbf{e}_0$. For $\langle \mathbf{y}, \mathbf{v}\rangle = 1$, note that $\mathbf{y}$ has $m$ nonzero elements, and $y(1)=1$ is already nonzero, thus there are $m-1$ remaining nonzero elements in $d-1$ components. The values of these elements can be $1$ or $-1$. Therefore
\begin{eqnarray}
	\left| \left\{\mathbf{y}\in \mathcal{Y}|\langle \mathbf{y}, \mathbf{v}\rangle = 1 \right\}\right| &=& \binom{d-1}{m-1} 2^{m-1}.
\end{eqnarray}

For $\langle \mathbf{y}, \mathbf{v}\rangle = 0$, $m$ nonzero elements are allocated in $d-1$ components. The values of these $m$ elements are either $1$ or $-1$. Therefore
\begin{eqnarray}
	\left|\left\{\mathbf{y}\in \mathcal{Y}|\langle \mathbf{y}, \mathbf{v}\rangle = 0 \right\}\right| &=& \binom{d-1}{m} 2^m.	
\end{eqnarray}
Finally, the number of cases with $\langle \mathbf{y}, \mathbf{v}\rangle=-1$ is the same as such number with $\langle \mathbf{y}, \mathbf{v}\rangle = 1$. Thus
\begin{eqnarray}
	\left| \left\{\mathbf{y}\in \mathcal{Y}|\langle \mathbf{y}, \mathbf{v}\rangle = -1\right\} \right| &=& \binom{d-1}{m-1} 2^{m-1}.
\end{eqnarray}
From \eqref{eq:pq}, among all elements in $\mathcal{Y}$, elements with $\langle \mathbf{y}, \mathbf{v}\rangle = 1$ are assigned with probability $e^\epsilon/\Omega$, while other elements are assigned with $1/\Omega$. Hence we can derive the normalizer in \eqref{eq:norm}.
\end{proof}

From Lemma \ref{lem:normalize}, it can be observed that with $m=1$, $\Omega = e^\epsilon+2d-1$. The randomizer \eqref{eq:pq} reduces to kRR over $2d$ alphabet size. In this case, our method is actually the same as the simple generalization that converts $\mathcal{V}$ to $\{0,\ldots, 2^{j+1}-1\}$. If $m>1$, it can be shown that the normalizer of our method \eqref{eq:norm} is smaller than the simple generalization method, indicating that our method yields a smaller output space and is thus more efficient.

\subsection{Aggregation}
 Given some $j\in \{0,1,\ldots, J\}$, with $\mathbf{Y}_i$ for all $i\in S_j$, the goal is to obtain $a_{jk}$, $k=0,\ldots, 2^j-1$ that is as close to $a_{jk}^*$ as possible. To achieve an unbiased aggregation, we need to calculate some important probabilities. Define
\begin{eqnarray}
	p:= \text{P}(Y(k) = 1|v(k)=1),
\end{eqnarray}
and
\begin{eqnarray}
	q:=\text{P}(Y(k)=1|v(k)=0).
\end{eqnarray}
Then we show the following lemma.
\begin{lem}\label{lem:pq}
	Let $d=2^j$. For a randomizer $Q:\mathcal{V}\rightarrow \mathcal{Y}$ with pmf specified in \eqref{eq:pq}, the following facts hold:
	
	(1) The values of $p$ is
	\begin{eqnarray}
		p = \frac{1}{\Omega} \binom{d-1}{m-1} 2^{m-1} e^\epsilon;
		\label{eq:p}
	\end{eqnarray}
	
	(2) If $m=1$, then
	\begin{eqnarray}
		q = \frac{1}{\Omega};
		\label{eq:q1}
	\end{eqnarray}
	If $m\geq 2$, then
	\begin{eqnarray}
		q &=& \frac{1}{\Omega}\binom{d-2}{m-2} 2^{m-2}e^\epsilon + \frac{1}{\Omega} \binom{d-2}{m-2}2^{m-2} \nonumber\\
		&&+ \frac{1}{\Omega} \binom{d-2}{m-1}2^{m-1};
		\label{eq:q2}
	\end{eqnarray}
	
	(3) 
	\begin{eqnarray}
		\text{P}(Y(k)=1|v(k)=-1) = pe^{-\epsilon}.
	\end{eqnarray}
\end{lem}

The proof of Lemma \ref{lem:pq} is shown in Appendix \ref{sec:pq}. Lemma \ref{lem:pq} is important since our aggregation rules are based on these probabilities.

Recall that $\mathbf{Y}_i$, $i\in S_j$ are the outputs of the randomizer $Q$ specified in \eqref{eq:pq}, with input $\mathbf{v}_i$, $i\in S_j$. Based on Lemma \ref{lem:pq}, the expected value is
\begin{eqnarray}
	\mathbb{E}[\mathbf{Y}_i|\mathbf{v}_i] = p(1-e^{-\epsilon}) \mathbf{v}_i.
	\label{eq:ey}
\end{eqnarray}
Therefore, to achieve an unbiased estimation of $a_{jk}^*$, we let
\begin{eqnarray}
	a_{jk} = \frac{2^{j/2}}{n_jp(1-e^{-\epsilon})} \sum_{i\in S_j} Y_i(k).
	\label{eq:ajk}
\end{eqnarray}
Regarding $a_{jk}$, we have the following lemma.
\begin{lem}\label{lem:biasvar}
	Let $d=2^j$. Then
	
	(1) $a_{jk}$ is unbiased, i.e.
	\begin{eqnarray}
		\mathbb{E}[a_{jk}] = a_{jk}^*;
	\end{eqnarray}
	(2) The variance given the sample allocation $S_0,\ldots, S_J$ can be expressed as
	\begin{eqnarray}
		&&\sum_{k=0}^{2^j - 1}\Var[a_{jk}|S_j]\nonumber\\
		&&= \frac{2^j}{n_j}\left[\frac{1+e^{-\epsilon}}{p(1-e^{-\epsilon})^2} - 1+\frac{q(d-1)}{p^2(1-e^{-\epsilon})^2}\right].
	\end{eqnarray}
	(3) The overall variance is bounded by
	\begin{eqnarray}
		\sum_{k=0}^{2^j - 1}\Var[a_{jk}] \leq  \frac{2^j}{n_j}\left[\frac{1+e^{-\epsilon}}{p(1-e^{-\epsilon})^2} +\frac{q(d-1)}{p^2(1-e^{-\epsilon})^2}\right].
		\label{eq:overall}
	\end{eqnarray}
\end{lem}
The proof of Lemma \ref{lem:biasvar} is shown in Appendix \ref{sec:biasvar}. To minimize the variance in \eqref{eq:overall}, we can let $m$ be
\begin{eqnarray}
	m^*=\arg\min_m \left[\frac{1+e^{-\epsilon}}{p(1-e^{-\epsilon})^2} +\frac{q(d-1)}{p^2(1-e^{-\epsilon})^2}\right].
	\label{eq:mopt}
\end{eqnarray}

\subsection{Implementation and Parameter Selection}



Now we discuss the implementation of the randomizer \eqref{eq:pq}. A simple way is to randomly generate $\mathbf{y}\in \mathcal{Y}$ with uniform probability at first, and then accept the result with probability $1$ if $\langle \mathbf{y}, \mathbf{v}\rangle=1$, and $e^{-\epsilon}$ if $\langle \mathbf{y}, \mathbf{v}\rangle = 0$. If it is not accepted, then draw $\mathbf{y}\in \mathcal{Y}$ repeatedly, until some vector $\mathbf{y}$ is accepted. The time complexity of each sampling step is $O(d)$. Moreover, one needs to draw $\mathbf{y}\in \mathcal{Y}$ for $O(e^\epsilon)$ times before acceptance. Therefore, the overall time complexity for each sample is $O(de^\epsilon)$. This method is suitable for small $\epsilon$. However, under large $\epsilon$, the sampling becomes inefficient. 

To overcome the drawback of such simple implementation, we now propose a fast implementation method. Let $j$ be the nonzero element of an encoded vector $\mathbf{v}$, i.e. $v(j) \in \{1,-1\}$. Then let
\begin{eqnarray}
	Y(j) = \left\{
	\begin{array}{ccc}
		\begin{array}{ccc}
			v(j) &\text{with probability} & p\\
			0 &\text{with probability} & 1-p(1+e^{-\epsilon})\\
			-v(j) &\text{with probability} & pe^{-\epsilon}
		\end{array}
	\end{array}
	\right.\hspace{-3mm}
	\label{eq:yj}
\end{eqnarray}
After that, if $Y(j)=1$, then we pick $m-1$ elements randomly from remaining $d-1$ elements. If $Y(j) = 0$, then we pick $m$ elements from these $d-1$ elements. The procedure is stated precisely in \autoref{alg:perturb}.

\begin{algorithm}[t]
	\caption{Perturbation}\label{alg:perturb}
	\textbf{Input:} The encoded vector $\mathbf{v}\in \mathcal{V}$\\
	\textbf{Output:} Perturbation output $\mathbf{Y} \in \mathcal{Y}$
	\textbf{Parameter:} $d$, $m$
	\begin{algorithmic}[1]
		\STATE Select $Y(j)$ according to \eqref{eq:yj}
		\IF{$Y(j)=0$}
		\STATE Randomly select $m$ elements $j_1,\ldots, j_m$ from $\{0,\ldots, 2^j-1\}\setminus \{j\}$ without replacement
		\STATE $Y(j_k)=1$ for $k=1,\ldots, m$
		\ELSE
		\STATE Randomly select $m-1$ elements $j_1,\ldots, j_{m-1}$ from $\{0,\ldots, 2^j - 1\}\setminus \{j\}$ without replacement
		\STATE $Y(j_k) = 1$ for $k=1,\ldots, m-1$
		\ENDIF
		\RETURN $\mathbf{Y}=(Y(0), \ldots, Y(2^j - 1)$
	\end{algorithmic}
\end{algorithm}

It can be shown that Algorithm \ref{alg:perturb} satisfies \eqref{eq:pq}, with time complexity $O(d)$. We omit the detailed proof here.

\subsection{Post-Processing}

Post-processing is common for frequency estimation methods \cite{wang2019locally}. By post-processing, we hope to achieve normalized and nonnegative frequency estimation. The above algorithm gives $a_{jk}$, which is an estimate of $a_{jk}^*$. After calculating $a_{jk}$ for all $j=0, \ldots, J$, the raw construction of the final pdf is

\begin{eqnarray}
	f_J(x) = 1+\sum_{j=0}^J \sum_{k=0}^{2^j - 1}a_{jk} \psi_{jk}(x).
	\label{eq:fj}
\end{eqnarray}

From the wavelet function \eqref{eq:haarjk}, $\int_0^1 \psi_{jk}(x) dx = 0$. Therefore, the above construction already ensures that $f_J$ is normalized, i.e. $\int_0^1 f_J(x) = 1$. However, the raw construction does not ensure the non-negativity. It is possible that $f_J(x)<0$ for some $x$. 

Our solution is to reconstruct $f_J(x)$ iteratively. Firstly, we let $f_0(x)=1$ for $x\in [0,1]$. In the $j$-th iteration, given $f_{j-1}(x)$ that was already reconstructed in previous iterations, We then update $a_{jk}$ as follows:
\begin{eqnarray}
	a_{jk}\leftarrow \Clip(a_{jk}, a_{\max}),
\end{eqnarray}
in which $\Clip(u, r) = \max(-r, \min(u, r))$ clips the absolute value, and 
\begin{eqnarray}
	a_{\max} = 2^{-j/2}\min(f_{j-1}(k2^{-j}), f_{j-1}((k+1)2^{-j})).
\end{eqnarray}
Such process is repeated from $j=1,\ldots, J$. It can be proved by induction that the above procedure yields a non-negative pdf, such that $f_J(x) \geq 0$ for all $x\in [0,1]$.

\subsection{Discussion}

\mypara{Complexity} Suppose that $n_j$ samples are allocated into the $j$-th group. The encoding, perturbation and aggregation mechanisms proposed in this section requires $O(n_j 2^j)$ time. These mechanisms are implemented over $j=0,1,\ldots, J$, thus the overall time complexity is $O(n 2^J/J)$. It appears to have an exponential dependence over $J$. However, according to the discussion in \autoref{sec:overall}, we pick $J=\lceil \log_2 n/2\rceil$, which has only a logarithmic dependence on $n$. Therefore, the overall time complexity is $O(n^{3/2}/\ln n)$.

\mypara{Other wavelet functions} In this paper, we use Haar wavelet basis. It is also possible to use other wavelet basis functions. The advantage of Haar wavelet is simpler encoding. Recall \autoref{sec:encoding}, the encoded variable $\mathbf{v}_i$ can take only $2^{j+1}$ values. Such a small output range leads to easier design of perturbation protocols. If we use other wavelet functions, then from \eqref{eq:encode}, each elements in $\mathbf{v}_i$ takes continuous numerical values, then we may use some high dimensional mean estimators to estimate $a_{jk}^*$ under $\epsilon$-LDP \cite{wang2019collecting,duchi2018minimax}, which is less efficient than the LDP mechanisms with Haar wavelet basis.

\section{Theoretical Analysis}\label{sec:theory}
This section provides an analysis of the proposed method. The goal of theoretical analysis is to provide a rigorous validation of the advantage of our new method compared with existing methods. Moreover, the theoretical analysis will also provide guidelines about the selection of hyperparameter $J$ and the allocation of samples in $S_0,\ldots, S_J$.

This work focuses on the Wasserstein distance $W(F_J, F)$, in which $F_J$ is the cdf of our reconstructed distribution using $a_{jk}$, $j=0,1,\ldots, J$, $k=0,1,\ldots, 2^j-1$. To be more precise,
\begin{eqnarray}
	F_J(x)=\int_0^x f_J(u) du,
	\label{eq:FJ}
\end{eqnarray}
in which $f_J$ is defined in \eqref{eq:fj}. $F$ is the real cdf defined in \eqref{eq:cdfreal}. 

From the triangle inequality of Wasserstein distance,

\begin{eqnarray}
	W(F_J, F)\leq W(F_J^*, F)+W(F_J, F_J^*).
	\label{eq:Wdec}
\end{eqnarray}
in which $F_J^*$ is the cdf of the optimal $J$-th order wavelet expansion \eqref{eq:FJstar}. $W(F_J^*, F)$ can be viewed as the approximation error, while $W(F_J, F_J^*)$ can be viewed as the estimation error.

\subsection{The Approximation Error}
\begin{lem}\label{lem:wdist}
	The Wasserstein distance between reconstructed cdf $F_J^*$ and the ground truth $F$ is bounded by
	\begin{eqnarray}
		W(F_J^*, F)\leq 2^{-(J+1)}.
	\end{eqnarray}
\end{lem}
\begin{proof}
	From Lemma \autoref{lem:approx}, $F_J^*(k2^{-(J+1)}) = F(k2^{-(J+1)})$. Hence, for any $x\in [k2^{-(J+1)}, (k+1)2^{-(J+1)}]$, $F(k2^{-(J+1)})\leq F_J^*(x)\leq F((k+1))\leq F((k+1)2^{-(J+1)})$. Therefore,
	\begin{eqnarray}
		&&\int |F_J^*(x)-F(x)|dx\nonumber\\ &=& \sum_{k=0}^{2^{J+1} - 1} \int_{k2^{-(J+1)}}^{(k+1)2^{-(J+1)}} |F_J^*(x)-F(x)|dx\nonumber\\
		&\leq & \sum_{k=0}^{2^{J+1}} 2^{-(J+1)}\left[ F((k+1)2^{-(J+1)}) - F(k2^{-(J+1)})\right] \nonumber\\
		&=& 2^{-(J+1)}.
	\end{eqnarray}
\end{proof}

Lemma \ref{lem:wdist} suggests that with the increase of $J$, $F_J^*$ is a closer approximate of $F$. If the distribution satisfies some smoothness properties, the bound may be further improved. However, in the current work, we do not make any additional assumption on the distribution of samples.

\subsection{The Estimation Error}
Now we bound
\begin{eqnarray}
	W(F_J, F_J^*) = \int_0^1 |F_J(x)-F_J^*(x)|dx.
\end{eqnarray}
The analysis of $W(F_J, F_J^*)$ will help us to make a better allocation of samples in $S_j$, $j=1,\ldots, J$, which aims at minimizing $W(F_J, F_J^*)$. To begin with, define
\begin{eqnarray}
	V_j := \min_m 2^j \left[\frac{1+e^{-\epsilon}}{p(1-e^{-\epsilon})^2} + \frac{q(d-1)}{p^2(1-e^{-\epsilon})^2}\right].
	\label{eq:Vj}
\end{eqnarray}
From \eqref{eq:overall}, it can be found that when $m$ is tuned optimally,
\begin{eqnarray}
	\sum_{k=0}^{2^j - 1} \Var[a_{jk}] \leq  \frac{V_j}{n_j}.
\end{eqnarray}
Then $W(F_J, F_J^*)$ can be bounded with $V_j$, $j=0,\ldots, J$.
\begin{lem}\label{lem:Wbound}
	The estimation error can be bounded by
	
	\begin{eqnarray}
		\mathbb{E}[W(F_J, F_J^*)] &\leq & \sqrt{\sum_{j=0}^J 2^{-(2j+2)}\frac{V_j}{n_j}}
	\end{eqnarray}
\end{lem}
\begin{proof}
	From \eqref{eq:FJ} and \eqref{eq:FJstar},
	\begin{eqnarray}
		F_J(x)-F_J^*(x) = \sum_{j=0}^J \sum_{k=0}^{2^j - 1}(a_{jk} - a_{jk}^*)\int_0^x \psi_{jk}(u) du.
	\end{eqnarray}
	Regarding $\int_0^x \psi_{jk}(u) du$, according to \eqref{eq:haarjk}, there are two facts. Firstly, for each $j$, $\int_0^x \psi_{jk}(u) du \neq 0$ only if $k=\lfloor 2^j x\rfloor$. Secondly, the bound
	$\left|\int_0^x \psi_{jk}(u) du\right|\leq 2^{-\frac{j}{2} - 1}$ holds.
	Hence, define $k_0(x)=\lceil 2^j x\rceil$, then
	\begin{eqnarray}
		F_J(x)-F_J^*(x)=\sum_{j=0}^J (a_{j,k_0(x)}-a_{j,k_0(x)}^*)\int_0^x \psi_{jk}(u) du.
	\end{eqnarray}
	Therefore
	\begin{eqnarray}
		\Var[F_J(x)] \leq \sum_{j=0}^J 2^{-(j+2)}\Var[a_{j,k_0(x)}].
	\end{eqnarray}
	Our construction ensures that $a_{jk}$ is unbiased. Hence
	\begin{eqnarray}
		\mathbb{E}[W(F_J, F_J^*)] &=& \int_0^1 \mathbb{E}[|F_J(x)-F_J^*(x)|]dx\nonumber\\
		&\leq & \int_0^1 \sqrt{\Var[F_J(x)]} dx\nonumber\\
		&\leq & \sqrt{\int_0^1 \Var[F_J(x)]dx}.
		\label{eq:wbound1}
	\end{eqnarray}
	Note that
	\begin{eqnarray}
		&&\int_0^1 \Var[F_J(x)]dx\nonumber\\
		&=& \int_0^1 \sum_{j=0}^J 2^{-(j+2)} \Var[a_{jk_0(x)}] dx\nonumber\\
		&=&\sum_{j=0}^J 2^{-(j+2)}\sum_{k=0}^{2^j - 1}\int_{k2^{-j}}^{(k+1)2^{-j}} \Var[a_{jk}]dx\nonumber\\
		&=&\sum_{j=0}^J 2^{-(2j+2)} \sum_{k=0}^{2^j - 1}\Var[a_{jk}]\nonumber\\
		&\leq & \sum_{j=0}^J 2^{-(2j+2)}\frac{V_j}{n_j}.
		\label{eq:var}
	\end{eqnarray}
	Lemma \ref{lem:Wbound} can then be proved by combining \eqref{eq:wbound1} and \eqref{eq:var}.
\end{proof}
It remains to bound $V_j$. From \eqref{eq:Vj}, using Lemma \ref{lem:normalize} and Lemma \ref{lem:pq}, we can get the following results.
\begin{lem}\label{lem:vbound}
	(1) If $\epsilon<1$, then $V_j=O(2^{2j}/\epsilon^2)$;
	
	(2) If $\epsilon> j$, then $V_j=O(2^j)$;
	
	(3) If $1\leq \epsilon\leq j$, then $V_j=O(2^j + 2^{2j}/e^\epsilon)$.
\end{lem}

The proof of Lemma \ref{lem:vbound} is shown in Appendix \ref{sec:vbound}. Recall that for frequency estimation problem, the variance is $O(d/(n\epsilon^2))$ for $\epsilon\leq 1$, and $O(d/(ne^\epsilon))$ for $\epsilon > 1$ \cite{ye2018optimal,zhao2025attack}. Compared with the bound for frequency estimation (with $d=2^j$), the bound of $V_j$ is different in several aspects. Firstly, as $\psi_{jk}(X_i)$ take values in $\{2^{j/2}, 0, -2^{j/2}\}$ instead of $\{1,0,-1\}$, the variance has an additional $2^j$ factor. Secondly, samples in $S_j$ are randomly selected from the original dataset. $V_j$ includes the variance caused by the randomness of sample allocation.

\subsection{Sample Allocation}\label{sec:allocation}
Now we discuss the sample allocation in $S_j$, $j=0,\ldots, J$. According to Lemma \ref{lem:Wbound}, we solve the following optimization problem:

\begin{equation*}
	\begin{aligned}
		& \underset{n_0,\ldots, n_J}{\text{minimize}}
		& & \sum_{j=0}^J 2^{-(2j+2)} \frac{V_j}{n_j} \\
		& \text{subject to}
		& & \sum_{j=0}^J n_j = n.
	\end{aligned}
\end{equation*}
From the above optimization problem, we assign samples according to the following rule:
\begin{eqnarray}
	n_j = \left\lfloor\frac{2^{-j} \sqrt{V_j}}{\sum_{l=0}^J 2^{-l} \sqrt{V_l}} n\right\rfloor.
	\label{eq:nj}
\end{eqnarray}
Due to the flooring operation, the sum of $n_j$ may be slightly less than $n$. The remaining samples can be allocated arbitrarily.

\subsection{The Overall Bound}\label{sec:overall}
Now we put the above analysis together, and get the following theorem.

\begin{thm}\label{thm:overall}
	With the allocation rule in \autoref{sec:allocation}, the Wasserstein distance is bounded by
	\begin{eqnarray}
		&&\mathbb{E}[W(F_J, F)]\nonumber\\
		&&=\left\{
		\begin{array}{ccc}
			O\left(\frac{J}{\epsilon\sqrt{n}}\right) + 2^{-(J+1)} &\text{if} & \epsilon<1\\
			O\left(\frac{1}{\sqrt{n}}+\frac{J}{\sqrt{n}e^\epsilon}\right)+2^{-(J+1)} &\text{if} & 1\leq \epsilon\leq \ln J\\
			O\left(\frac{1}{\sqrt{n}}\right) + 2^{-(J+1)} &\text{if} & \epsilon>\ln J.
		\end{array}
		\right.\nonumber\\
		\label{eq:total}
	\end{eqnarray}
\end{thm}
Theorem \autoref{thm:overall} can be proved by combination of Lemma \ref{lem:wdist}, \ref{lem:Wbound}, \ref{lem:vbound} and the allocation rule \eqref{eq:nj}. We omit detailed steps here. From Theorem \ref{thm:overall}, we can let $J=\lceil(\log_2 n)/2\rceil$, then the bound of total Wasserstein distance becomes $O(\ln n/(\epsilon\sqrt{n}))$ for small $\epsilon$, and $O(1/\sqrt{n})$ for large $\epsilon$.

If we put all samples into $S_J$, and let $S_j$ remains empty for $j=0,1,\ldots, J-1$, then our algorithm and theoretical results reduce to categorical frequency estimation methods. Therefore, compared with methods for categorical data, our wavelet expansion assign most samples to estimate coefficients of lower order. This prevents the probability mass from being misplaced too far away. Moreover, our analysis does not rely on any smoothness assumption of the data distribution. Compared with existing methods that randomizes samples in numerical domain, our method does not blur the sample distribution. Therefore, our new proposed wavelet expansion approach is more suitable to non-smooth distributions.
\section{Evaluation}\label{sec:numerical}
In this section, we show some numerical experiments to verify our proposed approach.

\subsection{Experimental Setup}
\mypara{Datasets} In this experiment, we use the following datasets.

\begin{itemize}
	\item Synthesized $Beta(5,2)$ dataset. The pdf of $Beta(\alpha, \beta)$ distribution is defined as
	\begin{eqnarray}
		f(x) = \frac{1}{\mathbb{B}(\alpha, \beta)}x^{\alpha - 1}(1-x)^{\beta - 1}, x\in [0,1]
	\end{eqnarray}
	in which $\mathbb{B}(\alpha, \beta)=\int_0^1 t^{\alpha-1}(1-t)^{\beta - 1}dt$ is the Beta function. In this experiment, we generate $N=100,000$ samples, which follows \cite{li2020estimating}. The $Beta(5,2)$ distribution can be viewed as an example of smooth distributions.
	
	\item Retirement. The San Francisco employee retirement plan data \cite{retirement} records the salary and benefits paid to city employees. There are $682,410$ samples in total, and we map all samples into $[0,1]$. For the Retirement dataset, the distribution is non-smooth.
	
	\item Income. This dataset comes from the American Community Survey \cite{ruggles2015integrated}. The values range from $0$ to $1,563,000$. We extract samples that are smaller than $2^{19}$ and map them into $[0,1]$. Income dataset is highly spiky, since people tend to report their income with precision up to hundreds or thousands. 
	
	\item Taxi pickup. This dataset comes from 2018 January New York Taxi data \cite{taxi}. This dataset contains the pickup time in a day. The samples are mapped into $[0,1]$. The total sample size is $8,760,687$. Since pickup events happen randomly within a day, the sample distribution is relatively smooth.
\end{itemize}

\mypara{Competitors} 
For all of these datasets, we compare the experimental results of our methods with the following baselines.
\begin{itemize}
	\item Categorical frequency oracle with binning. We divide the support $[0,1]$ into $d$ bins, such that the length of each bin is $h=1/d$. All samples are assigned to these bins, and then we just use OUE or kRR to estimate the distribution. According to the analysis in \cite{wang2017locally}, for a better performance, if $d<3e^\epsilon+2$, then we use kRR, otherwise we use OUE. In our experiments, we try $d=8$,$16$,$32$ and $64$, respectively. In the remainder of this section, we use Binning-$d$ to denote categorical frequency binning with $d$ bins. The final result is calibrated using the Norm-Sub technique \cite{wang2019locally}.
	
	\item Hierarchy histogram (HH). We follow the algorithm proposed in \cite{kulkarni2019answering}, which conducts a dyadic decomposition of the domain, which constructs a full $B$-ary tree with height $h$. Each layer of the tree is estimated separately.
	
	\item Convolution approach. We follow \cite{fang2023locally}, which converts the distribution estimation problem into a deconvolution algorithm. \cite{fang2023locally} proposed Direct Wiener (DW) and Improved Iterative Wiener (IIW) algorithms. In our experiments, we run both algorithms. The performances are very close between DW and IIW. In the figures, we show the better one among DW and IIW.
	
	\item SW with EM. We follow the square wave mechanism in \cite{li2020estimating} and run its Algorithm1 for post-processing.
\end{itemize}

\mypara{Metrics} We evaluate the results using both Wasserstein and KS distances, which are defined in \eqref{eq:Wdf} and \eqref{eq:KSdf}, respectively. To calculate the values of Wasserstein and KS distances, we construct $M=256$ grids. For each grid, we calculate the real empirical cdf from the raw data, as well as the estimated value. Then the Wasserstein distance is calculated by averaging the absolute difference between empirical cdf and its estimation, and the KS distance is the maximum absolute distance.

\mypara{Implementation} For each dataset, each method and each predetermined privacy budget, we repeat the experiment for $100$ times and take the average values for comparison.

\subsection{Overall Performance}

\begin{figure*}[h!]
	\includegraphics[width=\linewidth]{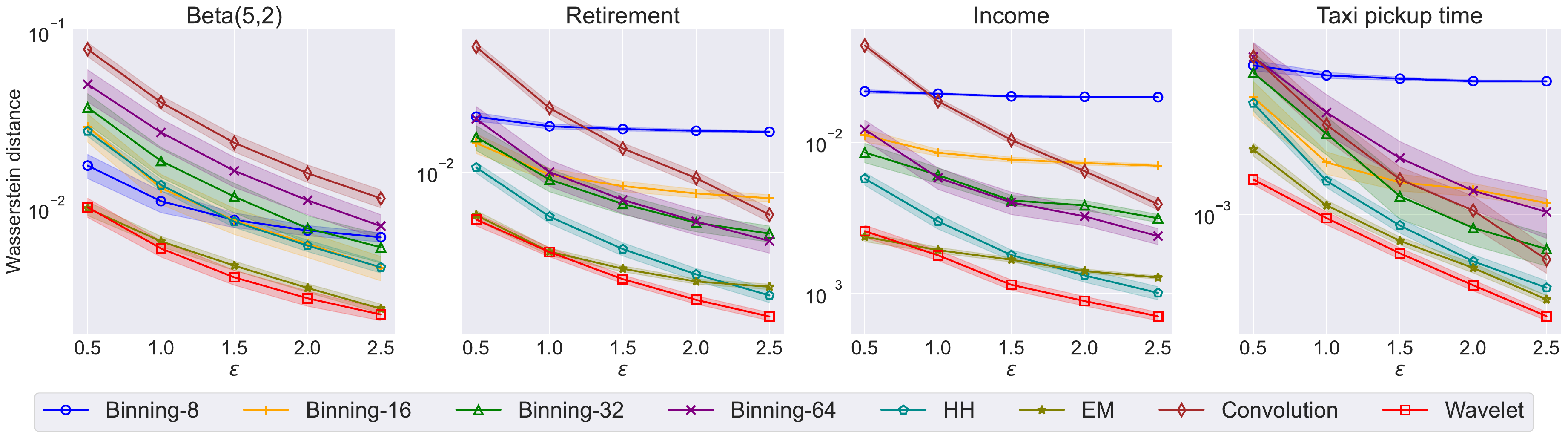}
	\caption{Wasserstein distances for different datasets.}\label{fig:wdist}
\end{figure*}
\begin{figure*}[h!]
	\includegraphics[width=\linewidth]{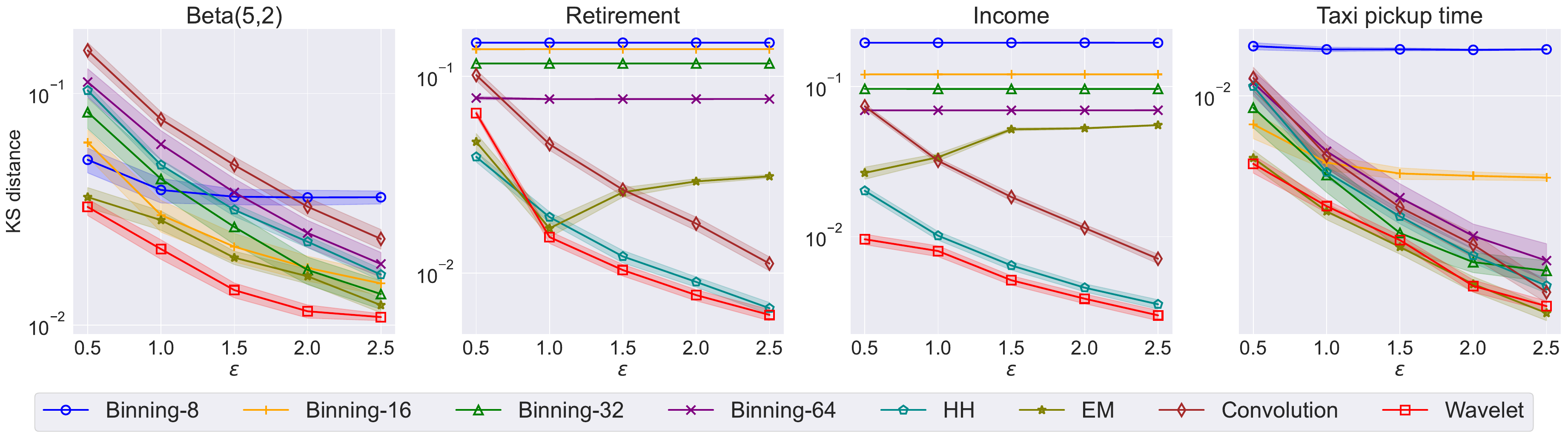}
	\caption{Kolmogorov-Smirnov distances for different datasets.}\label{fig:ksdist}
\end{figure*}

\autoref{fig:wdist} shows the Wasserstein distance between estimated distributions and the ground truth. According to the experimental results, for categorical frequency oracles, when the sample size $n$ is small (such as the $Beta(5,2)$ dataset) and $\epsilon$ is small, a small $d$ is more preferable, as a large bin size reduces the estimation variance. On the contrary, For large $n$ and $\epsilon$, frequency oracles with a large $d$ performs better, due to relatively smaller bias. However, even under the optimal choice of $d$, the performances are still far from optimal. HH \cite{cormode2021frequency} and SW with EM \cite{li2020estimating} significantly improve the performance compared with categorical frequency oracles.

Compared with these baseline methods, our wavelet expansion method performs consistently better. In particular, compared with SW with EM, the advantage of our method is especially obvious for non-smooth datasets, including Retirement and Income, with relatively large $\epsilon$. 

\autoref{fig:ksdist} shows the KS distance. Our wavelet expansion approach still performs better than all existing methods. For datasets with smooth distributions such as $Beta(5,2)$ and Taxi pickup time, the KS distance converges to $0$ with the increase of $\epsilon$ in general. However, for non-smooth datasets, including the Retirement and the Income datasets, which have many spikes, the KS distances of both categorical frequency oracles and SW with EM do not converge to zero. Intuitively, as discussed in \cite{li2020estimating}, this happens because these methods make the distribution more smooth. The smoothing operation leads to large error under KS distance, since it measures the maximum absolute difference between cdfs. Compared with existing methods, our wavelet expansion approach does not rely on the smoothness of the distribution. When the distribution contains some spikes, the wavelet method can preserve them, thus our method exhibits significantly better performance. 

\subsection{Validation of $J$'s Optimality}

 Now we test whether our parameter selection rule derived from the theoretical analysis is also optimal in numerical experiments. In the following experiment, we calculate the Wasserstein distance between estimated distribution and the ground truth with respect to different $J$. We run experiments for all four datasets, in which $J$ ranges from $1$ to $10$. Theoretical bounds are calculated by \eqref{eq:total}. Empirical results are averaged over $100$ random trials. The results are shown in  \autoref{fig:jsearch}. The optimal values of $J$ are highlighted in green and blue dashed curves, for theoretical and empirical results, respectively.

\begin{figure}[h!]
	\includegraphics[width=\linewidth]{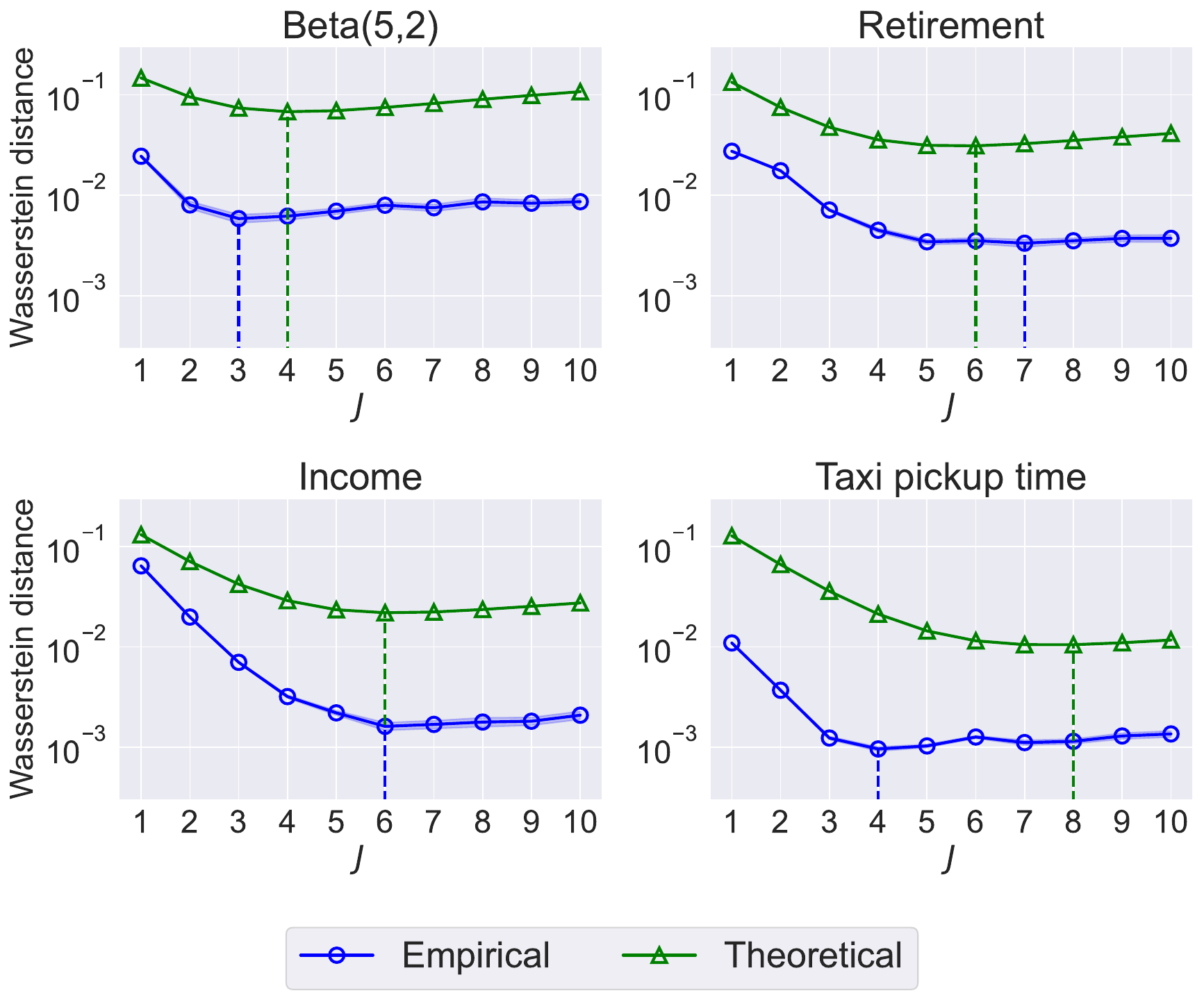}
	\caption{The performance of wavelet expansion from both theoretical analysis and experiments.}\label{fig:jsearch}
\end{figure}

From \autoref{fig:jsearch}, the empirical results are all lower than the theoretical bounds, which validates the correctness of our theoretical analysis. For $Beta(5,2)$, Retirement and Income datasets, The optimal values of $J$ obtained from empirical results are the same as or very close to those obtained from minimizing the theoretical bound \eqref{eq:total}. For the Taxi pickup time dataset, the optimal $J$ from experiments is $4$, significantly smaller than the theoretical results $8$. Our explanation is that while wavelet expansion with smaller $J$ can neglect some details of the distribution and thus results in large bias, for datasets with smooth distributions (such as Taxi pickup time), such additional bias is relatively small. As a result, a smaller $J$ may perform better in practice. 

\subsection{Impact of Smoothness of Distribution}

As discussed earlier in this paper, for SW with EM, as well as all other methods that add noise to samples and attempt to recover the original distribution, the noise blurs the original distribution, which is relatively more harmful for non-smooth distributions. In this experiment, we evaluate the impact of spikes on the performance. We run experiments with the following distribution:
\begin{eqnarray}
	f(x) = \left\{
	\begin{array}{ccc}
		1.5 &\text{if} & \lfloor x/h\rfloor \text{ is odd}\\
		0.5 &\text{if} & \lfloor x/h\rfloor \text{ is even,}
	\end{array}
	\right.
\end{eqnarray}
in which $h$ is $1/2$, $1/4$, $1/8$ and $1/16$, respectively.

\begin{figure*}
	\includegraphics[width=\linewidth]{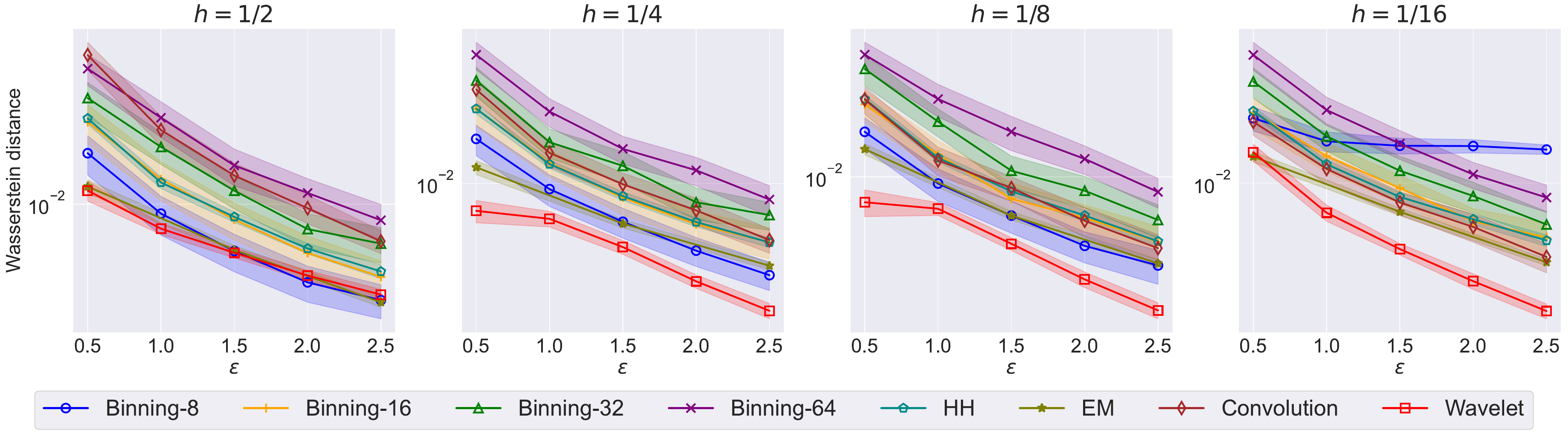}
	\caption{Wasserstein distances with respect to different $h$.}\label{fig:comb}
\end{figure*}

The results are shown in \autoref{fig:comb}. The results show that binning-based approaches do not yield good performance in general. In particular, with $h=1/16$ and $8$ bins (shown in the blue curve in the fourth subfigure), the Wasserstein distance does not converge to zero. SW with EM improves the performance. Compared with SW with EM, our method has made further progress. 

The relative advantage of our method varies among different $h$. In particular, with $h=1/2$, the performance of our new wavelet expansion approach is slightly better than SW with EM, but the advantage is not very obvious. With $h=1/4$ and $h=1/8$, our method begins to exhibit some advantage over SW with EM. With $h=1/16$, which is highly spiky, our new method significantly outperforms existing methods. These results agree with our theoretical analysis.

\subsection{Case Study: Range Query}
\begin{figure*}[h!]
	\includegraphics[width=\linewidth]{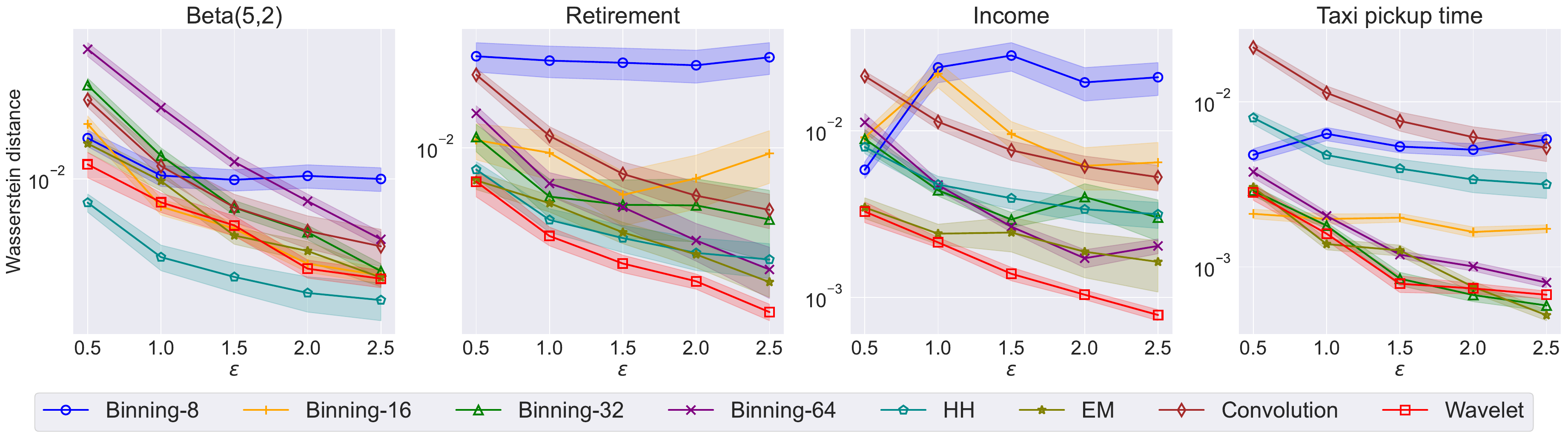}
	\caption{MAE of range query with $\alpha = 0.2$.}\label{fig:rq0.2}
\end{figure*}
\begin{figure*}[h!]
	\includegraphics[width=\linewidth,height=0.27\linewidth]{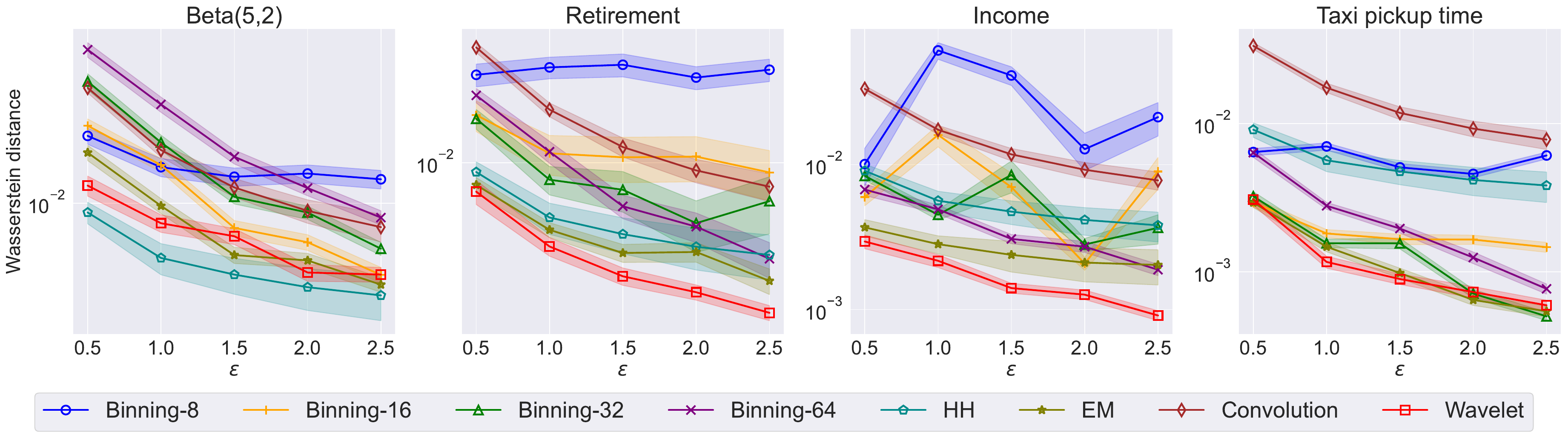}
	\caption{MAE of range query with $\alpha = 0.4$.}\label{fig:rq0.4}
\end{figure*}
Range query is an important application of distribution estimation. In this experiment, we fix $\alpha$ to be $0.2$ and $0.4$, and then randomly generate $a\sim \Unif([0,1-\alpha])$, $b=a+\alpha$. We then calculate the mean absolute error of the range query:
\begin{eqnarray}
	MAE = \left|\hat{F}(b)-\hat{F}(a)-\frac{1}{n}\sum_{i=1}^n \mathbf{1}(a\leq X_i\leq b)\right|,
\end{eqnarray}
in which $\hat{F}(b)-\hat{F}(a)$ is the estimated fraction of samples within $[a,b]$, while $(1/n)\sum_{i=1}^n \mathbf{1}(a\leq X_i\leq b)$ is the ground truth. 

\autoref{fig:rq0.2} and \autoref{fig:rq0.4} show the result with $\alpha = 0.2$ and $\alpha = 0.4$, respectively. The results show that our newly proposed wavelet expansion approach performs best in most cases. For datasets with smooth distributions, including the $Beta(5,2)$ dataset and the Taxi pickup time dataset, SW with EM has similar performance to our method. These results agree with our analysis. The noise mechanism SW makes the distribution smoother, and is thus suitable for smooth distributions. On the contrary, our method does not rely on the smoothness of distribution. Retirement and Income datasets have many spikes that need to be preserved, thus our method performs significantly better for these two datasets.

In general, the numerical experiments verify the effectiveness of our proposed method. For all datasets and all metrics investigated in this paper, the performances of our newly proposed wavelet expansion method are consistently better in most cases. In particular, compared with SW with EM and convolution algorithms, our method is better than preserving spikes in the distribution, thus the advantage is especially obvious for nonsmooth distributions.

\section{Related Work}\label{sec:related}

For estimating the distribution of numerical data, a simple baseline is categorical frequency oracles with binning. We first provide an overview of categorical frequency oracles as well as calibration strategies. We then briefly review several methods to handle ordinal or numerical data. Finally, we discuss other works related to wavelet expansion.

\subsection{Categorical Frequency Oracles}
The frequency estimation problem under LDP can be traced back to the random response technique \cite{warner1965randomized}. In recent years, many frequency oracles have been proposed, such as Rappor \cite{erlingsson2014rappor}, kRR \cite{kairouz2016discrete}, OUE and OLH \cite{wang2017locally}. kRR is suitable for large $\epsilon$, while OUE is suitable for small $\epsilon$. Several works attempt to achieve optimal frequency estimation at all privacy levels. \cite{ye2018optimal} provides the subset selection approach, and proves its minimax optimality. \cite{acharya2019hadamard} proposes the Hadamard response method, which further reduces the communication cost compared with the subset selection. Recently, several other categorical frequency oracles have been proposed. Fast local hashing \cite{cormode2021frequency} reduces the computation cost of OLH, at the cost of slight accuracy loss. Projective geometry response \cite{feldman2022private} is a new method that achieves similar accuracy to OUE and reduces the computation cost. \cite{fang2023locally} proposes a convolution approach to further improve the performance of frequency estimation for smooth datasets. There are also some hierarchy-based methods, such as \cite{hay2010boosting,qardaji2013understanding}. Moreover, there are also some works focusing on the heavy-hitter problem under LDP, which is closely related to frequency estimation \cite{bassily2015local,bassily2017practical}. There are also several other works that focus on frequency estimation over set-valued data \cite{wang2018privset}. Moreover, some works extend the analysis to user-level data \cite{acharya2023discrete,zhao2024learning,zhao2024huber}

\subsection{Calibration}
The estimated frequency may not sum up to $1$, and probably contains negative values. \cite{lee2015maximum,bassily2019linear} proposes a maximum likelihood method to correct the frequency estimate with alternating direction method of multipliers (ADMM). To cope with the computational issue of ADMM in high dimensionality, \cite{mckenna2019graphical} proposes a gradient descent-based method. \cite{wang2019locally} proposes a correction method to achieve consistent frequency estimation. It converts negative and small estimated frequency values to zero and ensure that the final sum is $1$. \cite{jia2019calibrate} designs a calibration approach, which assumes that the true frequencies follow a certain type of distribution with unknown parameters, and then conducts parameter estimation. Other normalization techniques have been developed and analyzed in \cite{hay2010boosting,wang2018privtrie}

\subsection{Handling Ordinal or Numerical Data}
When the data is ordinal or numerical, the simplest method is to bucketize the data into bins and just apply categorical frequency oracles. \cite{wang2017local} proposes an estimator under a refined DP definition, which requires two output distributions to be indistinguishable only when inputs are close to each other. There are also some other mechanisms that add noise to each numerical sample, such as stochastic rounding \cite{ding2017collecting,duchi2018minimax} and piecewise mechanism \cite{wang2019collecting}.

\subsection{Wavelets in DP}

Without privacy requirements, \cite{weed2019estimation} and \cite{niles2022minimax} have shown that the distribution of a random variable can be estimated via wavelet expansion, which achieves minimax optimal convergence rate of the Wasserstein distance between the estimated distribution and the ground truth. \cite{xiao2010differential} applies wavelet to frequency estimation under DP. Our work is different from \cite{xiao2010differential} in two aspects. Firstly, the method in \cite{xiao2010differential} was designed under central DP, while our method is designed for local DP. Secondly, \cite{xiao2010differential} focuses on categorical data, while our work is designed for numerical data, and the metric is changed to Wasserstein and KS distance. This work was then improved in \cite{xiao2014differentially}.

\nocite{zhao2025enhancing,ye2023stateful}

\section{Conclusion}\label{sec:conc}

In this paper, we have proposed a wavelet expansion approach to estimating numerical distributions under LDP. For numerical distributions, Wasserstein and KS distances are more appropriate. These metrics penalize heavily on misplacement of probability mass for a large distance. The proposed method works by prioritizing the estimation of low-order coefficients, and thus ensures accurate estimation at macroscopic level. We have also provided a theoretical analysis, which gives a convergence guarantee and provides guidelines on the parameter selection. Finally, we have run numerical experiments to validate our proposed approach. The result shows that our new proposed method consistently outperforms existing methods, including categorical frequency oracles, as well as SW with EM. The advantage of our method is especially obvious for non-smooth distributions.

\section*{Acknowledgements}

The work of Puning Zhao was supported in party by Open Research Projects of the Key Laboratory of Blockchain Technology and Data Security of the Ministry of Industry and Information Technology for the year 2025 (KT20250015), and the Open Research Fund of
The State Key Laboratory of Blockchain and Data Security, Zhejiang University. The work of Zhikun Zhang was supported in part by the NSFC under Grants No. 62402431, 62441618, and Zhejiang University Education Foundation Qizhen Scholar Foundation. The work of Li Shen was supported by the NSFC Grant (No. 62576364), Shenzhen Basic Research Project (Natural Science Foundation) Basic Research Key Project (NO. JCYJ20241202124430041), and   Major Project in Judicial Research from Supreme People's Court (NO. GFZDKT2024C08-3). The work of Shaowei Wang was supported by National Natural Science Foundation of China (No.62372120, 62102108), GuangDong Basic and Applied Basic Research Foundation (No.2022A1515010061), and Science and Technology Projects in Guangzhou (No.2025A03J3182). The work of Zhe Liu was supported by the National Natural Science Foundation of China (62132008, U22B2030), the Natural Science Foundation of Jiangsu Province (BK20220075).

\bibliographystyle{ieeetr}
\bibliography{wavelet}

\begin{thebibliography}{10}

\bibitem{dwork2006calibrating}
C.~Dwork, F.~McSherry, K.~Nissim, and A.~Smith, ``Calibrating noise to
  sensitivity in private data analysis,'' in {\em Theory of Cryptography: Third
  Theory of Cryptography Conference, TCC 2006, New York, NY, USA, March 4-7,
  2006. Proceedings 3}, pp.~265--284, Springer, 2006.

\bibitem{apple}
A.~D.~P. Team, ``Learning with privacy at scale,'' tech. rep., 2017.

\bibitem{erlingsson2014rappor}
{\'U}.~Erlingsson, V.~Pihur, and A.~Korolova, ``Rappor: Randomized aggregatable
  privacy-preserving ordinal response,'' in {\em Proceedings of the 2014 ACM
  SIGSAC conference on computer and communications security}, pp.~1054--1067,
  2014.

\bibitem{ding2017collecting}
B.~Ding, J.~Kulkarni, and S.~Yekhanin, ``Collecting telemetry data privately,''
  {\em Advances in Neural Information Processing Systems}, vol.~30, 2017.

\bibitem{yang2024local}
M.~Yang, T.~Guo, T.~Zhu, I.~Tjuawinata, J.~Zhao, and K.-Y. Lam, ``Local
  differential privacy and its applications: A comprehensive survey,'' {\em
  Computer Standards \& Interfaces}, vol.~89, p.~103827, 2024.

\bibitem{gu2019supporting}
X.~Gu, M.~Li, Y.~Cao, and L.~Xiong, ``Supporting both range queries and
  frequency estimation with local differential privacy,'' in {\em 2019 IEEE
  Conference on Communications and Network Security (CNS)}, pp.~124--132, IEEE,
  2019.

\bibitem{bassily2015local}
R.~Bassily and A.~Smith, ``Local, private, efficient protocols for succinct
  histograms,'' in {\em Proceedings of the forty-seventh annual ACM symposium
  on Theory of computing}, pp.~127--135, 2015.

\bibitem{kairouz2016discrete}
P.~Kairouz, K.~Bonawitz, and D.~Ramage, ``Discrete distribution estimation
  under local privacy,'' in {\em International Conference on Machine Learning},
  pp.~2436--2444, PMLR, 2016.

\bibitem{wang2017locally}
T.~Wang, J.~Blocki, N.~Li, and S.~Jha, ``Locally differentially private
  protocols for frequency estimation,'' in {\em 26th USENIX Security Symposium
  (USENIX Security 17)}, pp.~729--745, 2017.

\bibitem{acharya2019hadamard}
J.~Acharya, Z.~Sun, and H.~Zhang, ``Hadamard response: Estimating distributions
  privately, efficiently, and with little communication,'' in {\em The 22nd
  International Conference on Artificial Intelligence and Statistics},
  pp.~1120--1129, PMLR, 2019.

\bibitem{ye2018optimal}
M.~Ye and A.~Barg, ``Optimal schemes for discrete distribution estimation under
  locally differential privacy,'' {\em IEEE Transactions on Information
  Theory}, vol.~64, no.~8, pp.~5662--5676, 2018.

\bibitem{wang2020locally}
T.~Wang, M.~Lopuhaa-Zwakenberg, Z.~Li, B.~Skoric, and N.~Li, ``Locally
  differentially private frequency estimation with consistency,'' in {\em
  NDSS'20: Proceedings of the NDSS Symposium}, 2020.

\bibitem{fang2023locally}
H.~Fang, L.~Chen, Y.~Liu, and Y.~Gao, ``Locally differentially private
  frequency estimation based on convolution framework,'' in {\em 2023 IEEE
  Symposium on Security and Privacy (SP)}, pp.~2208--2222, IEEE, 2023.

\bibitem{christin2016privacy}
D.~Christin, ``Privacy in mobile participatory sensing: Current trends and
  future challenges,'' {\em Journal of Systems and Software}, vol.~116,
  pp.~57--68, 2016.

\bibitem{li2020estimating}
Z.~Li, T.~Wang, M.~Lopuha{\"a}-Zwakenberg, N.~Li, and B.~{\v{S}}koric,
  ``Estimating numerical distributions under local differential privacy,'' in
  {\em Proceedings of the 2020 ACM SIGMOD International Conference on
  Management of Data}, pp.~621--635, 2020.

\bibitem{li2021trade}
M.~Li, Y.~Tian, J.~Zhang, D.~Fan, and D.~Zhao, ``The trade-off between privacy
  and utility in local differential privacy,'' in {\em 2021 International
  Conference on Networking and Network Applications (NaNA)}, pp.~373--378,
  IEEE, 2021.

\bibitem{du2024numerical}
L.~Du, P.~Cheng, L.~Zheng, X.~Lian, L.~Chen, W.~Xi, and W.~Ni, ``Numerical
  estimation of spatial distributions under differential privacy,'' {\em arXiv
  preprint arXiv:2412.06541}, 2024.

\bibitem{feldman2024instance}
V.~Feldman, A.~McMillan, S.~Sivakumar, and K.~Talwar, ``Instance-optimal
  private density estimation in the wasserstein distance,'' {\em arXiv preprint
  arXiv:2406.19566}, 2024.

\bibitem{rabin2012wasserstein}
J.~Rabin, G.~Peyr{\'e}, J.~Delon, and M.~Bernot, ``Wasserstein barycenter and
  its application to texture mixing,'' in {\em Scale Space and Variational
  Methods in Computer Vision: Third International Conference, SSVM 2011,
  Ein-Gedi, Israel, May 29--June 2, 2011, Revised Selected Papers 3},
  pp.~435--446, Springer, 2012.

\bibitem{massey1951kolmogorov}
F.~J. Massey~Jr, ``The kolmogorov-smirnov test for goodness of fit,'' {\em
  Journal of the American statistical Association}, vol.~46, no.~253,
  pp.~68--78, 1951.

\bibitem{warner1965randomized}
S.~L. Warner, ``Randomized response: A survey technique for eliminating evasive
  answer bias,'' {\em Journal of the American statistical association},
  vol.~60, no.~309, pp.~63--69, 1965.

\bibitem{zhang2019wavelet}
D.~Zhang and D.~Zhang, ``Wavelet transform,'' {\em Fundamentals of image data
  mining: Analysis, Features, Classification and Retrieval}, pp.~35--44, 2019.

\bibitem{villani2009optimal}
C.~Villani {\em et~al.}, {\em Optimal transport: old and new}, vol.~338.
\newblock Springer, 2009.

\bibitem{wang2018locally}
T.~Wang, N.~Li, and S.~Jha, ``Locally differentially private frequent itemset
  mining,'' in {\em 2018 IEEE Symposium on Security and Privacy (SP)},
  pp.~127--143, IEEE, 2018.

\bibitem{wang2019local}
S.~Wang, L.~Huang, Y.~Nie, X.~Zhang, P.~Wang, H.~Xu, and W.~Yang, ``Local
  differential private data aggregation for discrete distribution estimation,''
  {\em IEEE Transactions on Parallel and Distributed Systems}, vol.~30, no.~9,
  pp.~2046--2059, 2019.

\bibitem{cover1999elements}
T.~M. Cover, {\em Elements of information theory}.
\newblock John Wiley \& Sons, 1999.

\bibitem{wang2019locally}
T.~Wang, M.~Lopuha{\"a}-Zwakenberg, Z.~Li, B.~Skoric, and N.~Li, ``Locally
  differentially private frequency estimation with consistency,'' {\em arXiv
  preprint arXiv:1905.08320}, 2019.

\bibitem{wang2019collecting}
N.~Wang, X.~Xiao, Y.~Yang, J.~Zhao, S.~C. Hui, H.~Shin, J.~Shin, and G.~Yu,
  ``Collecting and analyzing multidimensional data with local differential
  privacy,'' in {\em 2019 IEEE 35th International Conference on Data
  Engineering (ICDE)}, pp.~638--649, IEEE, 2019.

\bibitem{duchi2018minimax}
J.~C. Duchi, M.~I. Jordan, and M.~J. Wainwright, ``Minimax optimal procedures
  for locally private estimation,'' {\em Journal of the American Statistical
  Association}, vol.~113, no.~521, pp.~182--201, 2018.

\bibitem{zhao2025attack}
P.~Zhao, Z.~Zhang, J.~Dong, J.~Wu, S.~Wang, Z.~Liu, and Y.~Gao, ``An
  attack-agnostic defense framework against manipulation attacks under local
  differential privacy,'' in {\em 2025 IEEE Symposium on Security and Privacy
  (SP)}, pp.~3858--3876, IEEE Computer Society, 2025.

\bibitem{retirement}
``{SF} employee compensation.''
  https://www.kaggle.com/datasets/san-francisco/sf-employee-compensation.

\bibitem{ruggles2015integrated}
S.~Ruggles, K.~Genadek, R.~Goeken, J.~Grover, and M.~Sobek, ``Integrated public
  use microdata series,'' 2015.

\bibitem{taxi}
``{TLC} trip record data.''
  https://www.nyc.gov/site/tlc/about/tlc-trip-record-data.page.

\bibitem{kulkarni2019answering}
T.~Kulkarni, ``Answering range queries under local differential privacy,'' in
  {\em Proceedings of the 2019 International Conference on Management of Data},
  pp.~1832--1834, 2019.

\bibitem{cormode2021frequency}
G.~Cormode, S.~Maddock, and C.~Maple, ``Frequency estimation under local
  differential privacy,'' {\em Proceedings of the VLDB Endowment}, vol.~14,
  no.~11, pp.~2046--2058, 2021.

\bibitem{feldman2022private}
V.~Feldman, J.~Nelson, H.~Nguyen, and K.~Talwar, ``Private frequency estimation
  via projective geometry,'' in {\em International Conference on Machine
  Learning}, pp.~6418--6433, PMLR, 2022.

\bibitem{hay2010boosting}
M.~Hay, V.~Rastogi, G.~Miklau, and D.~Suciu, ``Boosting the accuracy of
  differentially private histograms through consistency,'' {\em Proceedings of
  the VLDB Endowment}, vol.~3, no.~1, 2010.

\bibitem{qardaji2013understanding}
W.~Qardaji, W.~Yang, and N.~Li, ``Understanding hierarchical methods for
  differentially private histograms,'' {\em Proceedings of the VLDB Endowment},
  vol.~6, no.~14, pp.~1954--1965, 2013.

\bibitem{bassily2017practical}
R.~Bassily, K.~Nissim, U.~Stemmer, and A.~Guha~Thakurta, ``Practical locally
  private heavy hitters,'' {\em Advances in Neural Information Processing
  Systems}, vol.~30, 2017.

\bibitem{wang2018privset}
S.~Wang, L.~Huang, Y.~Nie, P.~Wang, H.~Xu, and W.~Yang, ``Privset: Set-valued
  data analyses with locale differential privacy,'' in {\em IEEE INFOCOM
  2018-IEEE Conference on Computer Communications}, pp.~1088--1096, IEEE, 2018.

\bibitem{acharya2023discrete}
J.~Acharya, Y.~Liu, and Z.~Sun, ``Discrete distribution estimation under
  user-level local differential privacy,'' in {\em International Conference on
  Artificial Intelligence and Statistics}, pp.~8561--8585, PMLR, 2023.

\bibitem{zhao2024learning}
P.~Zhao, L.~Shen, R.~Fan, Q.~Li, H.~Wu, J.~Wu, and Z.~Liu, ``Learning with
  user-level local differential privacy,'' {\em arXiv preprint
  arXiv:2405.17079}, 2024.

\bibitem{zhao2024huber}
P.~Zhao, L.~Lai, L.~Shen, Q.~Li, J.~Wu, and Z.~Liu, ``A huber loss minimization
  approach to mean estimation under user-level differential privacy,'' {\em
  Advances in Neural Information Processing Systems}, vol.~37,
  pp.~130018--130056, 2024.

\bibitem{lee2015maximum}
J.~Lee, Y.~Wang, and D.~Kifer, ``Maximum likelihood postprocessing for
  differential privacy under consistency constraints,'' in {\em Proceedings of
  the 21th ACM SIGKDD International Conference on Knowledge Discovery and Data
  Mining}, pp.~635--644, 2015.

\bibitem{bassily2019linear}
R.~Bassily, ``Linear queries estimation with local differential privacy,'' in
  {\em The 22nd International Conference on Artificial Intelligence and
  Statistics}, pp.~721--729, PMLR, 2019.

\bibitem{mckenna2019graphical}
R.~McKenna, D.~Sheldon, and G.~Miklau, ``Graphical-model based estimation and
  inference for differential privacy,'' in {\em International Conference on
  Machine Learning}, pp.~4435--4444, PMLR, 2019.

\bibitem{jia2019calibrate}
J.~Jia and N.~Z. Gong, ``Calibrate: Frequency estimation and heavy hitter
  identification with local differential privacy via incorporating prior
  knowledge,'' in {\em IEEE INFOCOM 2019-IEEE Conference on Computer
  Communications}, pp.~2008--2016, IEEE, 2019.

\bibitem{wang2018privtrie}
N.~Wang, X.~Xiao, Y.~Yang, T.~D. Hoang, H.~Shin, J.~Shin, and G.~Yu,
  ``Privtrie: Effective frequent term discovery under local differential
  privacy,'' in {\em 2018 IEEE 34th International Conference on Data
  Engineering (ICDE)}, pp.~821--832, IEEE, 2018.

\bibitem{wang2017local}
S.~Wang, Y.~Nie, P.~Wang, H.~Xu, W.~Yang, and L.~Huang, ``Local private ordinal
  data distribution estimation,'' in {\em IEEE INFOCOM 2017-IEEE Conference on
  Computer Communications}, pp.~1--9, IEEE, 2017.

\bibitem{weed2019estimation}
J.~Weed and Q.~Berthet, ``Estimation of smooth densities in wasserstein
  distance,'' in {\em conference on Learning Theory}, pp.~3118--3119, PMLR,
  2019.

\bibitem{niles2022minimax}
J.~Niles-Weed and Q.~Berthet, ``Minimax estimation of smooth densities in
  wasserstein distance,'' {\em The Annals of Statistics}, vol.~50, no.~3,
  pp.~1519--1540, 2022.

\bibitem{xiao2010differential}
X.~Xiao, G.~Wang, and J.~Gehrke, ``Differential privacy via wavelet
  transforms,'' {\em IEEE Transactions on knowledge and data engineering},
  vol.~23, no.~8, pp.~1200--1214, 2010.

\bibitem{xiao2014differentially}
X.~Xiao, ``Differentially private data release: Improving utility with wavelets
  and bayesian networks,'' in {\em Web Technologies and Applications: 16th
  Asia-Pacific Web Conference, APWeb 2014, Changsha, China, September 5-7,
  2014. Proceedings 16}, pp.~25--35, Springer, 2014.

\bibitem{zhao2025enhancing}
P.~Zhao, J.~Wu, Z.~Liu, L.~Shen, Z.~Zhang, R.~Fan, L.~Sun, and Q.~Li,
  ``Enhancing learning with label local differential privacy by vector
  approximation,'' in {\em 13th International Conference on Learning
  Representations, ICLR 2025}, pp.~46914--46931, International Conference on
  Learning Representations, ICLR, 2025.

\bibitem{ye2023stateful}
Q.~Ye, H.~Hu, K.~Huang, M.~H. Au, and Q.~Xue, ``Stateful switch: Optimized time
  series release with local differential privacy,'' in {\em IEEE INFOCOM
  2023-IEEE Conference on Computer Communications}, pp.~1--10, IEEE, 2023.

\end{thebibliography}

\appendices

\section{Proof of Lemma \ref{lem:approx}}\label{sec:approx}

We prove Lemma \autoref{lem:approx} by induction.

(1) For $J=0$. We need to show that $F_0^*(x) = F(x)$ for $x=0,1/2,1$. \eqref{eq:fjstar} ensures that $f_j^*$ is a valid pdf, thus $F_0^*(0) = 0 = F(0)$, $F_0^*(1)=1 = F(1)$. For $x=1/2$, note that
\begin{eqnarray}
	f_0^*(x)=1+a_{00}^* \psi_{00}(x),
\end{eqnarray}
with
\begin{eqnarray}
	a_{00}^* &=& \frac{1}{n}\sum_{i=1}^n \psi_{00}(X_i) \nonumber\\
	&=& \frac{1}{n}\sum_{i=1}^n \mathbf{1}\left(X_i<\frac{1}{2}\right)+\frac{1}{n}\sum_{i=1}^n \mathbf{1}\left(X_i\geq \frac{1}{2}\right).\nonumber\\
\end{eqnarray}
Then
\begin{eqnarray}
	F_0\left(\frac{1}{2}\right) &=& \int_0^\frac{1}{2} f_0^*(t) dt\nonumber\\
	&=& \int_0^\frac{1}{2}(1+a_{00}^* \psi_{00}(t)) dt\nonumber\\
	&=&\frac{1}{2}(1+a_{00}^*)\nonumber\\
	&=&\frac{1}{n}\sum_{i=1}^n \mathbf{1}\left(X_i<\frac{1}{2}\right) \nonumber\\
	&=& F\left(\frac{1}{2}\right).
\end{eqnarray}
(2) Suppose that Lemma \autoref{lem:approx} holds up to $J$. Then for the $J+1$-th expansion, we need to show that $F_{J+1}^*(x)=F(x)$ for $x=k2^{-(J+2)}$, $k=0,\ldots, 2^{J+2}$. We discuss the case with even and odd $k$ separately.

If $k$ is even, then 
\begin{eqnarray}
	F_{J+1}^*(x)&=&\int_0^x f_{J+1}^*(t) dt\nonumber\\
	&=&\int_0^x \left[1+\sum_{j=0}^{J+1}\sum_{k=0}^{2^j - 1}a_{jk}^* \psi_{jk}(t)\right] dt\nonumber\\
	&=& F_J^*(x)+\int_0^x a_{J+1,k}^* \psi_{J+1,k}(t) dt.
\end{eqnarray}
Recall the definition of $\psi$ in \eqref{eq:psijk}, for $k=0,1,\ldots, 2^{J+1} - 1$,
\begin{eqnarray}
	\int_{k2^{-(J+1)}}^{(k+1)2^{-(J+1)}}\psi_{J+1,k}(t) dt=0.
\end{eqnarray}
Therefore
\begin{eqnarray}
	\int_0^{k2^{-(J+1)}} \psi_{J+1, k}(t)dt = 0.
	\label{eq:int0}
\end{eqnarray}
Since $x=k2^{-(J+2)}=(k/2)2^{-(J+1)}$, with $k/2$ being an integer, we have
\begin{eqnarray}
	\int_0^x \psi_{J+1, k}(t) dt = 0.
\end{eqnarray}
Therefore
\begin{eqnarray}
	F_{J+1}^*(x) = F_J^*(x) = F(x),
\end{eqnarray}
in which the second equality holds by induction.

If $k$ is odd, then
\begin{eqnarray}
	&&F_{J+1}^*(x)\nonumber\\
	&=&F_J^*(x)+\int_0^x \sum_{k=0}^{2^{J+1} - 1}a_{J+1, k}^* \psi_{J+1, k}(t) dt\nonumber\\
	&\overset{(a)}{=} & F_J^*(x)+\int_{(k-1)2^{-(J+2)}}^{k2^{-(J+2)}}a_{J+1,k}^* \psi_{J+1,k}(t) dt\nonumber\\
	&\overset{(b)}{=} & F_J^*(x)+2^{-(J+2)}2^{\frac{J+1}{2}}\frac{1}{n}\sum_{i=1}^n \psi_{J+1,k}(X_i)\nonumber\\
	&\overset{(c)}{=} & \frac{1}{2}F_J^*((k-1)2^{-(J+2)})+\frac{1}{2}F_J^*((k+1)2^{-(J+2)})\nonumber\\
	&&+\frac{1}{2}\left[\frac{1}{n}\sum_{i=1}^n \mathbf{1}((k-1)2^{-(J+2)}\leq X_i<k2^{-(J+2)})\right.\nonumber\\
	&&\left.-\frac{1}{n}\sum_{i=1}^n \mathbf{1}(k2^{-(J+2)}\leq X_i<(k+1)2^{-(J+2)})\right]\nonumber\\
	&\overset{(d)}{=} & \frac{1}{2}F_J^*((k-1)2^{-(J+2)})+\frac{1}{2}F_J^*((k+1)2^{-(J+2)})\nonumber\\
	&&+\frac{1}{2}\left[2F(k2^{-(J+2)})-F((k-1)2^{-(J+2)})\right.\nonumber\\
	&&\left.-F((k+1)2^{-(J+2)})\right]\nonumber\\
	&=& F(k2^{-(J+2)})\nonumber\\
	&=& F(x).
\end{eqnarray}
(a) uses \eqref{eq:int0}. For odd $k$, $k-1$ is even, thus
\begin{eqnarray}
	\int_0^{(k-1)2^{-(J+2)}} a_{J+1,k}^* \psi_{J+1,k}(t) dt = 0.
\end{eqnarray}
(b) holds because $\psi_{J+1, k}(t) = 2^\frac{J+1}{2}$ for $t\in [(k-1)2^{-(J+2)}, k2^{-(J+2)}]$.

(c) holds because $f_J^*(x)$ is uniform in $[k2^{-(J+1)}, (k+1)2^{-(J+1)}]$. Therefore, for any $k=0,1,\ldots, 2^{J+1} - 1$, $F_J^*((k+1/2)2^{-(J+1)})=(F_J^*(k2^{-(J+1)}) + F_J^*((k+1)2^{-(J+1)}))/2$.

(d) holds because $k-1$ and $k+1$ are even, thus $F_J^*((k-1)2^{-(J+2)})=F((k-1)2^{-(J+2)})$, $F_J^*((k+1)2^{-(J+2)})=F((k+1)2^{-(J+2)})$. 

Now we have proved that $F_{J+1}^*(x)=F(x)$ for $x=k2^{-(J+2)}$ for both even and odd $k$. The proof of Lemma \autoref{lem:approx} is complete.
\section{Proof of Lemma \ref{lem:pq}}\label{sec:pq}

\textbf{Proof that \eqref{eq:pq} satisfies $\epsilon$-LDP.} From \eqref{eq:pq}, for any two vectors $\mathbf{v}$ and $\mathbf{v}'$,
\begin{eqnarray}
	\frac{p_Q(\mathbf{y}|\mathbf{v})}{p_Q(\mathbf{y}|\mathbf{v}')}\leq e^\epsilon.
\end{eqnarray}

\textbf{Calculation of $p$}. 
\begin{eqnarray}
	\text{P}(Y(k)=1|v(k)=1) &=& \sum_{\mathbf{y}\in \mathcal{Y}} \mathbf{1}(y(k)=1)\frac{e^\epsilon}{\Omega}\nonumber\\
	&=& \binom{d-1}{m-1} 2^{m-1} \frac{e^\epsilon}{\Omega}.
\end{eqnarray}

\textbf{Calculation of $q$}. Now we calculate $\text{P}(Y(k)=1|v(k)=0)$. Recall that $\mathbf{v}$ has only one nonzero element. Without loss of generality, suppose $v(l)=1$ for $l\neq j$. Then
\begin{eqnarray}
	&&\text{P}(Y(k)=1|v(k)=0, v(l)=1)\nonumber\\
	&=& \sum_{\mathbf{y}\in \mathcal{Y}} \mathbf{1}(y(k)=1, y(l)=1)\frac{e^\epsilon}{\Omega} \nonumber\\
	&&+ \sum_{\mathbf{y}\in \mathcal{Y}} \mathbf{1}(y(k)=1, y(l)=0)\frac{1}{\Omega}\nonumber\\
	&&+\sum_{\mathbf{y}\in \mathcal{Y}} \mathbf{1}(y(k)=1, y(l)=-1)\frac{1}{\Omega}.
\end{eqnarray}
By combination rules,
\begin{eqnarray}
	\sum_{\mathbf{y}\in \mathcal{Y}} \mathbf{1}(y(k)=1, y(l)=1) &=& \binom{d-2}{m-2} 2^{m-2},\\
	\sum_{\mathbf{y}\in \mathcal{Y}} \mathbf{1}(y(k)=1, y(l)=0) &=& \binom{d-2}{m-1} 2^{m-1},\\
	\sum_{\mathbf{y}\in \mathcal{Y}} \mathbf{1}(y(k)=1, y(l)=-1)&=&\binom{d-2}{m-2} 2^{m-2}.
\end{eqnarray}
Therefore
\begin{eqnarray}
	&&\text{P}(Y(k)=1|v(k)=0, v(l)=1) \nonumber\\
	&&= \binom{d-2}{m-2} 2^{m-2} \frac{e^\epsilon}{\Omega} +\binom{d-2}{m-2} 2^{m-2} \frac{1}{\Omega} \nonumber\\
	&&+ \binom{d-2}{m-1} 2^{m-1} \frac{1}{\Omega}.
\end{eqnarray}
The above equation holds for all $l\neq k$. Hence
\begin{eqnarray}
	\text{P}(Y(k)=1|v(k)=0) = q.
\end{eqnarray}
Finally, 
\begin{eqnarray}
	\text{P}(Y(k)=1|v(k)=-1)&=&\sum_{\mathbf{y}\in \mathcal{Y}} \mathbf{1}(y(k)=1)\frac{1}{\Omega} \nonumber\\
	&=& \frac{1}{\Omega}\binom{d-1}{m-1} 2^{m-1}\nonumber\\
	&=& pe^{-\epsilon}.
\end{eqnarray}
The proof of Lemma \ref{lem:pq} is complete.

\section{Proof of Lemma \ref{lem:biasvar}}\label{sec:biasvar}
\textbf{Proof of the unbiasedness of $a_{jk}$.} Recall \eqref{eq:ey},
\begin{eqnarray}
	\mathbb{E}[Y_i(k)] = p(1-e^{-\epsilon})v_i(k).
\end{eqnarray}
Thus
\begin{eqnarray}
	\mathbb{E}[a_{jk}] &=& \mathbb{E}\left[\frac{2^{j/2}}{n_j p(1-e^{-\epsilon})} \sum_{i\in S_j} p(1-e^{-\epsilon}) v_i(k)\right]\nonumber\\
	&=& \mathbb{E}\left[\frac{1}{n_j}\sum_{i\in S_j} \psi_{jk}(X_i)\right]\nonumber\\
	&=& \frac{1}{n}\sum_{i=1}^n \psi_{jk}(X_i)\nonumber\\
	&=& a_{jk}^*.
\end{eqnarray}
Therefore $a_{jk}$ is an unbiased estimate of $a_{jk}^*$.

\textbf{The variance of $a_{jk}$.} We use Lemma \ref{lem:pq}. If $v(k)=1$, then
\begin{eqnarray}
	&&\Var[Y(k)|v(k)=1] \nonumber\\
	&=& \mathbb{E}[Y^2(k) |v(k)=1] - (\mathbb{E}[Y(k)|v(k)=1])^2\nonumber\\
	&=& \text{P}(Y(k)=1|v(k)=1) + \text{P}(Y(k)=-1|v(k)=1)\nonumber\\
	&&\hspace{-3mm}- (\text{P}(Y(k)=1|v(k)=1) - \text{P}(Y(k)=-1|v(k)=1))^2\nonumber\\
	&=& p + pe^{-\epsilon} - (p-pe^{-\epsilon})^2.
\end{eqnarray}
If $v(k)=-1$, then
\begin{eqnarray}
	\Var[Y(k)|v(k)=-1] = p(1+e^{-\epsilon}) - p^2 (1-e^{-\epsilon})^2,
\end{eqnarray}
which is the same as the case of $v(k)=1$. 

If $v(k)=0$, then
\begin{eqnarray}
	\Var[Y(k)|v(k)=0] = q.
\end{eqnarray}

Therefore
\begin{eqnarray}
	&&\Var\left[\sum_{i\in S_j} Y_i(k)\right] \nonumber\\
	&=& \sum_{i\in S_j} \left[\mathbf{1}(|v_i(k)|=1) \left[p(1+e^{-\epsilon}) - p^2(1-e^{-\epsilon})^2 \right] \right.\nonumber\\
	&&\left.+ \mathbf{1}(v_i(k)=0)q\right]\nonumber\\
	&=& n_j f_{jk} [p(1+e^{-\epsilon}) - p^2(1-e^{-\epsilon})^2] + n_j (1-f_{jk})q,\nonumber\\
\end{eqnarray}
in which
\begin{eqnarray}
	f_{jk} = \frac{1}{n_j}\sum_{i\in S_j} \mathbf{1}(|v_i(k)|=1).
\end{eqnarray}
Recall Section \ref{sec:encoding}, $\mathbf{v}_i$ is $1$-sparse, thus
\begin{eqnarray}
	\sum_{k=0}^{2^j - 1}f_{jk} = 1.
\end{eqnarray}
Then taking the sum over all $k$ yields
\begin{eqnarray}
	\sum_{k=0}^{2^j - 1}\Var[a_{jk}|S_j] = \frac{2^j}{n_j}\left[\frac{1+e^{-\epsilon}}{p(1-e^{-\epsilon})^2}+\frac{q(d-1)}{p^2(1-e^{-\epsilon})^2}\right].
\end{eqnarray}

\textbf{Proof of overall variance.} Now we analyze the variance including the randomness caused by sample allocation. 
\begin{eqnarray}
	\Var[a_{jk}] = \mathbb{E}[\Var[a_{jk}|S_j]] + \Var[E[a_{jk}|S_j]].
\end{eqnarray}
The first term is already analyzed in our previous step. Now we bound the second term $\Var[E[a_{jk}|S_j]]$. Recall that $\psi_{jk}(X_i)\in \{-2^{j/2}, 0, 2^{j/2}\}$, thus
\begin{eqnarray}
	&&\Var[\psi_{jk}(X_i) \mathbf{1}(i\in S_j)] \nonumber\\
	&&= \left\{
	\begin{array}{ccc}
		2^j p_{ij}(1-p_{ij}) &\text{if} &|\psi_{jk}(X_i)|= 2^{j/2}\\
		0 &\text{if} & \psi_{jk}(X_i) = 0,
	\end{array}
	\right.
\end{eqnarray}
in which $p_{ij}:=\text{P}(i\in S_j)$.

Consider that $\psi_{jk}(X_i) \mathbf{1}(i\in S_j)$ are negatively correlated for different $i$, the total variance is not larger than the sum of their individual variances. Therefore
\begin{eqnarray}
	&&\Var\left[\sum_{i\in S_j} \psi_{jk}(X_i)\right]\nonumber\\
	&& \leq \sum_{i=1}^n \mathbf{1}(\psi_{jk}(X_i)\in \{-2^{j/2}, 2^{j/2}\})2^j \frac{n_j}{n}(1-\frac{n_j}{n}),\nonumber\\
\end{eqnarray}
and
\begin{eqnarray}
	\sum_{k=0}^{2^j - 1}\Var[E[a_{jk}|S_j]] &=& \sum_{k=0}^{2^j - 1}\Var\left[\frac{1}{n_j}\sum_{i\in S_j} \psi_{jk}(X_i)\right]\nonumber\\
	&=& \frac{1}{n_j^2}\sum_{i=1}^n 2^j \frac{n_j}{n}(1-\frac{n_j}{n})\nonumber\\
	&=& \frac{2^j}{n_j}(1-\frac{n_j}{n})\nonumber\\
	&\leq & \frac{2^j}{n_j}.
\end{eqnarray}
The proof of Lemma \ref{lem:biasvar} is complete.

\section{Proof of Lemma \ref{lem:vbound}}\label{sec:vbound}
	(1) If $\epsilon<1$, then let $m=d$. By \eqref{eq:norm}, 
	\begin{eqnarray}
		\Omega = 2^{d-1} (e^\epsilon +1),
	\end{eqnarray}
	and
	\begin{eqnarray}
		p = \frac{1}{\Omega} 2^{d-1} e^\epsilon = \frac{e^\epsilon}{e^\epsilon+1},
	\end{eqnarray}
	\begin{eqnarray}
		q = \frac{1}{\Omega} 2^{d-2} e^\epsilon + \frac{1}{\Omega} 2^{d-2} = \frac{1}{2}.
	\end{eqnarray}
	Thus $V_j = O(2^{2j}/\epsilon^2)$.
	
	(2) If $\epsilon>j$, we let $m=1$. Then $\Omega = e^\epsilon + 2d - 1$, and
	\begin{eqnarray}
		p = \frac{e^\epsilon}{e^\epsilon+2d-1}, q = \frac{1}{e^\epsilon + 2d - 1},
	\end{eqnarray} 
	thus $V_j=O(2^j)$.
	
	(3) If $1<\epsilon<j$, then let $m=[de^{-\epsilon}]$, in which $[\cdot]$ means number rounding. The result shows that $V_j = O(2^j + 2^{2j}/e^\epsilon)$.

\newpage 

\appendices 

\section{Meta-Review}

The following meta-review was prepared by the program committee for the 2026
IEEE Symposium on Security and Privacy (S\&P) as part of the review process as
detailed in the call for papers.

\subsection{Summary}
This paper proposes a wavelet-based technique for estimating numerical distributions under local differential privacy. The paper includes theoretical guarantees for the accuracy and computational efficiency of the approach, as well as experiments which show improved performance compared to prior methods.

\subsection{Scientific Contributions}
\begin{itemize}
	\item Provides a Valuable Step Forward in an Established Field
\end{itemize}

\subsection{Reasons for Acceptance}
\begin{enumerate}
	\item The paper makes a valuable contribution to the established field of local differentially private distribution estimation by extending beyond categorical data to numerical distributions, using a principled wavelet-based approach tailored to the Wasserstein and Kolmogorov–Smirnov error metrics. This addresses an important gap, as frequency and density estimation are core building blocks for many downstream tasks.
	\item Theoretical and empirical contributions are both strong: the paper provides rigorous guarantees for univariate distributions in Wasserstein distance, showing optimal parameter dependence, and presents thorough experiments demonstrating that the method outperforms categorical-data baselines combined with binning, as well as recent EM-based heuristics, with particular advantages on spiky distributions while maintaining competitive performance for smooth ones.
\end{enumerate}

\end{document}